\documentclass{article}
\usepackage[margin=1.2in]{geometry}

\usepackage{graphicx} % more modern
\usepackage{subfigure}

\usepackage{natbib}

\usepackage[ruled]{algorithm}
\usepackage{algorithmic}
\usepackage{xr-hyper}
\usepackage{hyperref}

\usepackage{morenotations}
\usepackage{theorem}
\usepackage{xr}
\usepackage{booktabs, multicol, multirow}
\usepackage{rotating}
\usepackage{afterpage}

\usepackage{amsmath}
\usepackage{amssymb}
\usepackage{amstext}

\usepackage{xcolor}
\hypersetup{
    colorlinks,
    linkcolor={blue!50!blue},
    citecolor={blue!50!blue},
    urlcolor={blue!50!blue}
}

\def\calS{\mathcal{S}}
\def\calD{\mathcal{D}}

\def\sign{\mathrm{sign}}

\def\tanh#1{\mathrm{tanh}}
\def\prox{\mathrm{prox}}

\newtheorem{lemma}[theorem]{Lemma}
\newcommand\meanop{\bm{\mu}}
\newcommand\R{\mathbb{R}}
\newcommand\spl{\textsc{spl}}
\newcommand\lol{\textsc{lol}}
\newcommand\dotp{\langle \bm{\theta}, \bm{x} \rangle}
\newcommand\dotpi{\langle \bm{\theta}, \bm{x}_i \rangle}
\newcommand\dotmu{\langle \bm{\theta}, \bm{\mu} \rangle}

\title{Loss factorization, weakly supervised learning \\ and label noise robustness}

\author{
\textbf{Giorgio Patrini} \\
Australian National University, NICTA \\
\texttt{giorgio.patrini@anu.edu.au}
\and
\textbf{Frank Nielsen} \\
\'Ecole Polytechnique, Sony Computer Science Laboratories Inc. \\
\texttt{nielsen@lix.polytechnique.fr}
\and
\textbf{Richard Nock} \\
NICTA, Australian National University \\
\texttt{richard.nock@nicta.com.au}
\and
\textbf{Marcello Carioni} \\
Max Planck Institute for Mathematics in the Sciences \\
\texttt{marcello.carioni@mis.mpg.de}}

\begin{document} 
\maketitle

\begin{abstract}
We prove that
the empirical risk of most well-known loss functions factors into a linear term
aggregating all labels with a term that is label free,
and can further be expressed by sums of the loss.
This holds true even for 
non-smooth, non-convex losses and in any
\textsc{rkhs}. The first term is a 
(kernel) mean operator --the focal quantity of this work-- which we characterize as the sufficient statistic for the labels.
The result tightens known generalization bounds and sheds new light on their interpretation.

Factorization has a direct application on weakly supervised learning.
In particular, we demonstrate that algorithms like \textsc{sgd} and proximal methods can be adapted with minimal
effort to handle weak supervision, once the mean operator has been estimated.
We apply this idea to learning with asymmetric noisy labels, connecting and extending prior work.
Furthermore, we show that most losses enjoy a
data-dependent (by the mean operator) form of noise robustness,
in contrast with known negative results.
\end{abstract}

\section{Introduction}

Supervised learning is by far the most effective application of the machine learning paradigm. However, there is a growing need of decoupling the success of the field from its topmost framework, often unrealistic in practice. In fact while the amount of available data grows continuously, its relative training labels --often derived by human effort-- become rare, and hence learning is performed with partially missing, aggregate-level and/or noisy labels. For this reason, \emph{weakly supervised learning} (\textsc{wsl}) has attracted much research. In this work, we focus on binary classification under weak supervision. Traditionally, \textsc{wsl} problems are attacked by designing \emph{ad-hoc} loss functions and optimization algorithms tied to the particular learning setting. Instead, we advocate to ``do not reinvent the wheel" and present an unifying treatment. In summary, we show that, under a mild decomposability assumption,

\emph{Any loss admitting a minimizing algorithm over fully labelled data, can also be minimized in \textsc{wsl} setting with provable generalization and noise robustness guarantees. Our proof is constructive: we show that a simple change in the input and of one line of code is sufficient.}

\begin{figure}[t]
\centering
\subfigure[logistic loss]{\label{fig:1-logistic} \includegraphics[width=.30\textwidth]{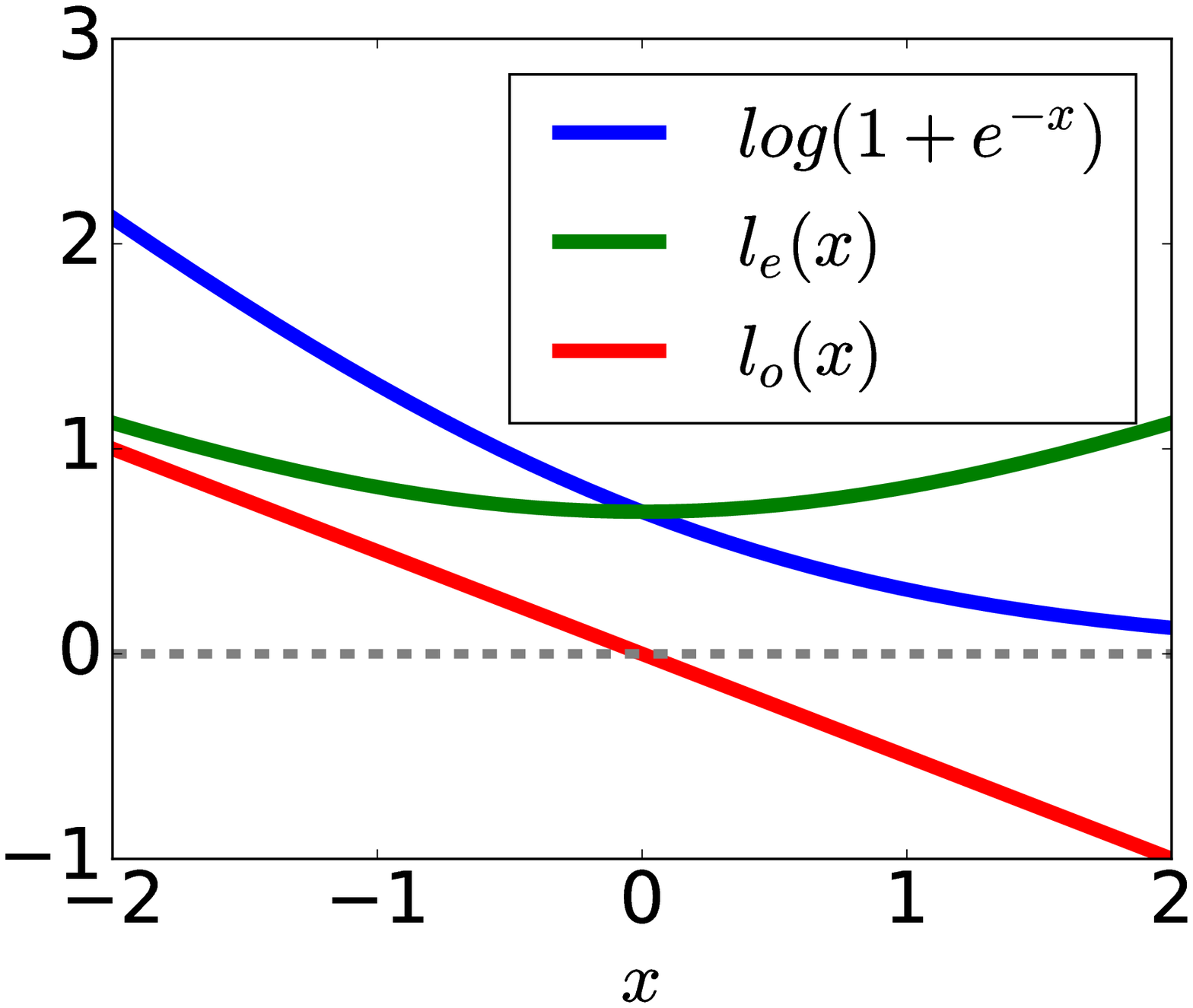}}
\subfigure[square loss]{\label{fig:1-square} \includegraphics[width=.30\textwidth]{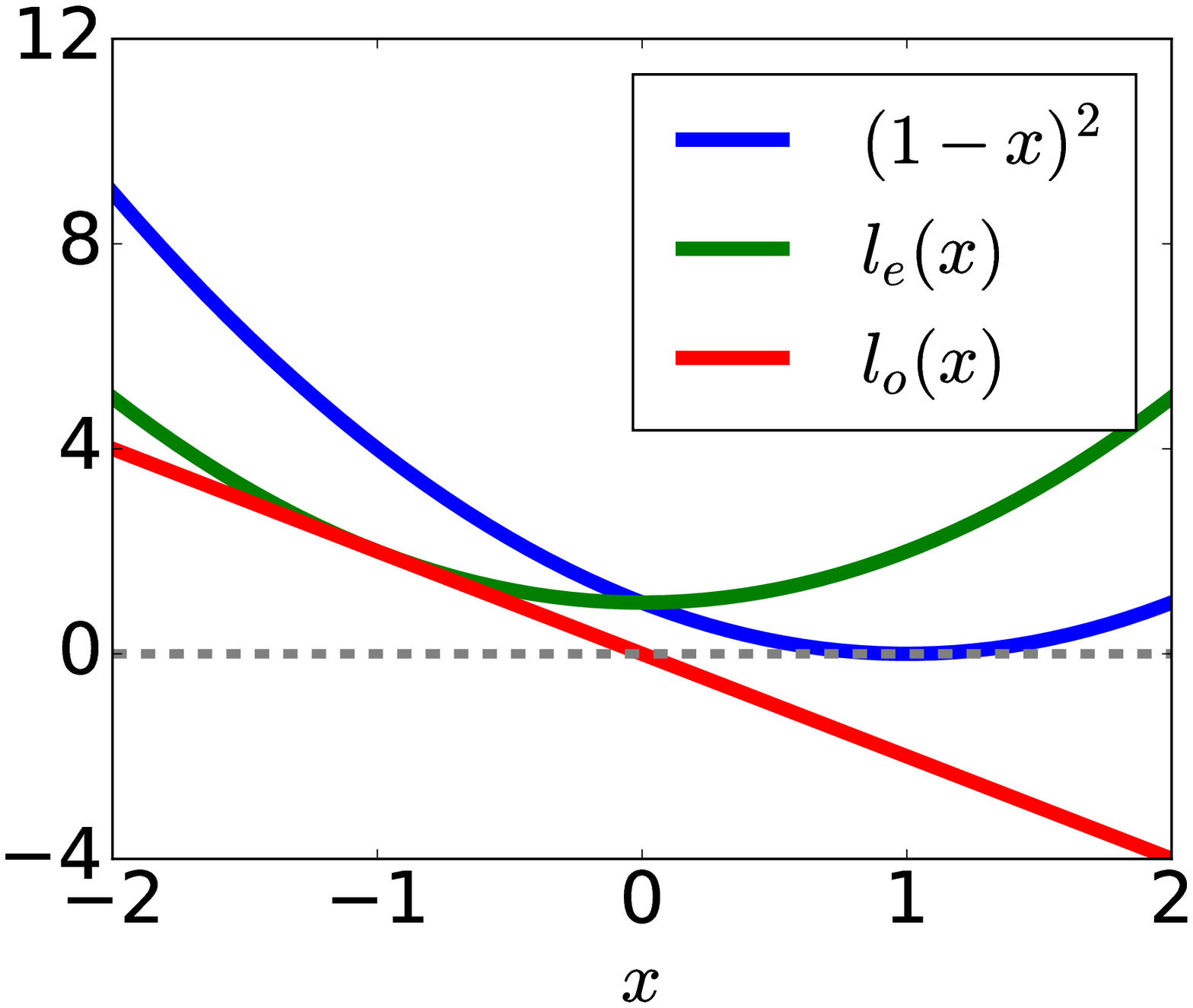}}
%\vspace{-10px}
\caption{Loss factorization: $l(x) = l_e(x) + l_o(x)$.}\label{fig:1}
%\vspace{-10px}
\end{figure}

\subsection{Contribution and related work}
%We overview our main contributions. 
We introduce \emph{linear-odd losses} (\lol s), a definition that not does demand smoothness or convexity, but that a loss $l$ is such that $l(x) - l(-x)$ is linear. Many losses of practical interest are such, \emph{e.g.} logistic and square. We prove a theorem reminiscent of Fisher-Neyman's factorization \citep{lcTO} of the exponential family which lays the foundation of this work: it shows how empirical $l$-risk factors (Figure \ref{fig:1}) in a label free term with another incorporating a \emph{sufficient statistic of the labels}, the mean operator. The interplay of the two components is still apparent on newly derived generalization bounds, that also improve on known ones \citep{kstOT}. Aside from factorization, the above linearity is known to guarantee convexity to certain losses used in learning on \emph{positive and unlabeled data} (\textsc{pu}) in the recent \citep{pnsCF}.

An advantage of isolating labels comes from applications on \textsc{wsl}. In this scenario, training labels are only partially observable due to an unknown noise process \citep{ggwDO, hgilWS}. For example, labels may be \emph{noisy} \citep{ndrtLW}, missing as with \emph{semi-supervision} \citep{cbzSS} and \textsc{pu} \citep{pnsCF}, or aggregated as it happens in \emph{multiple instance learning} \citep{dllST} and \emph{learning from label proportions} (\textsc{llp}) \citep{qsclEL}. As the success of those areas shows, labels are not strictly needed for learning. However, most \textsc{wsl} methods implicitly assumes that  labels must be recovered in training, as pointed out by \citep{jbAC}. Instead, sufficiency supports a more principled two-step approach: (1) estimate the mean operator from weakly supervised data and (2) plug it into any \lol~and resort to known procedures for empirical risk minimization (\textsc{erm}). Thus, (1) becomes the only technical obstacle in adapting an algorithm, although often easy to surpass. Indeed, this approach unifies a growing body of literature \citep{qsclEL, pnrcAN, rmwLW, gwlzRM}. As a showcase, we implement (2) by adapting stochastic gradient descent (\textsc{sgd}) to \textsc{wsl}. The upgrade only require to transform input by a ``double sample trick" and to sum the mean operator in the model update.

Finally, we concentrate on learning with asymmetric label noise. We connect and extend previous work of \citep{ndrtLW, rmwLW} by designing an unbiased mean operator estimator, for which we derive generalization bounds independent on the chosen \lol~and algorithm. Recent results \citep{lsRC} have shown that requiring the strongest form of robustness --on any possible noisy sample-- rules out most losses commonly used, and have drifted research focus on non-convex \citep{srLS, mmvOT, dvTL} or linear losses \citep{rmwLW}. Our approach is more pragmatic, as we show that \emph{any} \lol s enjoy an \emph{approximate} form of noise robustness which converges, on a data-dependent basis, to the strongest one. The mean operator is still central here, being the data-dependent quantity that shapes the bound. The theory is validated with experimental analysis, for which we call the adapted \textsc{sgd} as a black box.

Elements of this work appeared in an early version \citep{pnnBW}, mostly interested in elucidating the connection between loss factorization and $\alpha$-label differential privacy \citep{khSC}.

Next, Section \ref{sec:setting} settles notations and background. Section \ref{sec:factorization} states the Factorization Theorem. Sections \ref{sec:weakly} and \ref{sec:noisy} focus on \textsc{wsl} and noisy labels. Section \ref{sec:concl} discusses the paper. Proofs and additional results are given in Appendix.

\section{Learning setting and background}\label{sec:setting}

We denote vectors in bold as $\bm{x}$ and $1\{p\}$ the indicator of $p$ being $\textsc{true}$. We define $[m] \defeq \{1, \dots, m\}$ and $[x]_+ \defeq \max(0, x)$. In binary classification, a learning sample $\mathcal{S} = \{(\bm{x}_i, y_i), i \in [m] \}$ is a sequence of (observation, label) pairs, the examples, drawn from an unknown distribution $\calD$ over $\mathcal{X} \times \mathcal{Y}$, with $\mathcal{X} \subseteq \mathbb{R}^d$ and $\mathcal{Y} = \{-1,1\}$. Expectation (or average) over $(\bm{x}, y)\sim \calD$ ($\calS$) is denoted as $\expect_{\calD}$ ($\expect_{\calS}$).

Given a hypothesis $h \in \mathcal{H}, h: \mathcal{X} \rightarrow \mathbb{R}$, a loss is any function $l: \mathcal{Y} \times \mathbb{R} \rightarrow \mathbb{R}$. A loss gives a penalty $l(y, h(\bm{x}))$ when predicting the value $h(\bm{x})$ and the observed label is $y$. We consider \emph{margin losses}, \emph{i.e.} $l(y,h(\bm{x})) = l(yh(\bm{x}))$ \citep{rwCB}, which are implicitly symmetric: $l(yh(\bm{x}))=l(-y \cdot (-h(\bm{x})))$. For notational convenience, we will often use a generic scalar argument $l(x)$. Examples are $01$ loss $1\{x <0\}$, logistic loss $\log(1 + e^{-x})$, square loss $(1 - x)^2$ and hinge loss $[1-x]_+$.

The goal of binary classification is to select a hypothesis $h \in \mathcal{H}$ that generalizes on $\calD$. That is, we aim to minimize the \emph{true risk} on the $01$ loss $R_{\calD, 01}(h) \defeq \expect_{\calD}[1\{yh(\bm{x}) <0\}]$. In practice, we only learn from a finite learning sample $\calS$ and thus minimize the \emph{empirical $l$-risk} $R_{\calS, l} (h) \defeq \expect_{\mathcal{S}} [l(yh(\bm{x}))] = \frac{1}{m}\sum_{i \in [m]} l(y_ih(\bm{x}_i))$, where $l$ is a tractable upperbound of $01$ loss.

Finally, we discuss the meaning of \textsc{wsl} --and in particular of weakly supervised \emph{binary classification}. The difference with the above is at training time: we learn on a sample $\tilde{\calS}$ drawn from a noisy distribution $\tilde{\calD}$ that may flip, aggregate or suppress labels, while observations are the same. Still, the purpose of learning is unchanged: to minimize the true risk. A rigorous definition will not be relevant in our study.

\subsection{Background: exponential family and logistic loss}\label{sec:exp-fam}

Some background on the exponential family is to come.
We can learn a binary classifier fitting a model in the conditional exponential family parametrized by $\bm{\theta}$: $p_{\bm{\theta}}(y | \bm{x}) = \exp (\langle \bm{\theta}, y \bm{x} \rangle - \log \sum_{y \in \mathcal{Y}} \exp \langle \bm{\theta}, y \bm{x} \rangle)$, with $y$ random variable. The two terms in the exponent are the log-partition function and the sufficient statistic $y_i\bm{x}_i$, which fully summarizes one example $(\bm{x}, y)$. The Fisher-Neyman theorem \citep[Theorem 6.5]{lcTO} gives a sufficient and necessary condition for sufficiency of a statistic $T(y)$: the probability distribution factors in two functions, such that ${\bm{\theta}}$ interacts with the $y$ only through $T$:
$$p_{\bm{\theta}}(y) = g_{\bm{\theta}}(T(y)) g'(y)\:\:.$$
 In our case, it holds that $g'(y| \bm{x}) = 1$, $T(y| \bm{x}) = y\bm{x}$ and $g_{\bm{\theta}}(\cdot| \bm{x}) = \exp (\langle \bm{\theta}, \cdot \rangle\ - \log \sum_{y \in \mathcal{Y}} \exp (\langle \bm{\theta}, y \bm{x} \rangle)$, since the value of $y$ is not needed to compute $g_{\bm{\theta}}$. This shows how $y\bm{x}$ is indeed sufficient (for $y$). Now, under the \emph{i.i.d.} assumption, the log-likelihood of $\bm{\theta}$ is (the negative of)
\begin{align}
&\sum_{i=1}^m \log \sum_{y \in \mathcal{Y}} e^{y \dotpi}
- \sum_{i=1}^m \langle \bm{\theta}, y_i\bm{x}_i \rangle \label{eq:first}\\
&= \sum_{i=1}^m \log \sum_{y \in \mathcal{Y}} e^{y \dotpi} - 
\sum_{i=1}^m \log e^{ y_i \dotpi} \nonumber \\
&= \sum_{i=1}^m \log \left( \frac{e^{\dotpi} + e^{-\dotpi}} 
{e^{y_i\dotpi}} \right) \nonumber \\
&= \sum_{i=1}^m \log \left( 1 + e^{-2y_i\dotpi}\right)\:\:. \label{eq:second-last}
\end{align}
 Step (\ref{eq:second-last}) is true because $y \in \mathcal{Y}$. At last, by At last, by re-parameterizing $\bm{\theta}$ and normalizing, we obtain the empirical risk of logistic loss. Equation (\ref{eq:first}) shows how the loss splits into a linear term aggregating the labels and another, label free term. We translate this property for classification with \textsc{erm}, by transferring the ideas of sufficiency and factorization to a wide set of losses including the ones of \citep{pnrcAN}.

\section{Loss factorization and sufficiency}\label{sec:factorization}

The  linear term just encountered in logistic loss integrates a well-studied statistical object.
\begin{definition} The (empirical) \emph{mean operator} of a learning sample $\mathcal{S}$ is
$\meanop_{\calS} \defeq \expect_{\calS} \left[ y \bm{x} \right] \:\:.$
\end{definition}
We drop the $\calS$ when clear by the context. The name \emph{mean operator}, or mean map, is borrowed from the theory of Hilbert space embedding \citep{qsclEL}\footnote{We decide to keep the lighter notation of linear classifiers, but nothing should prevent the extension to non-parametric models, exchanging $\bm{x}$ with an implicit feature map $h(\bm{x})$.}. Its importance is due to the injectivity of the map --under conditions on the kernel-- which is used in applications such as two-sample and independence tests, feature extraction and covariate shift \citep{sgssAH}. Here, $\meanop$ will play the role of sufficient statistic for labels \emph{w.r.t.} a set of losses.
\begin{definition}
A function $T(\mathcal{S})$ is said to be a sufficient statistic for a variable $z$ \emph{w.r.t.} a set of losses $\mathcal{L}$ and a hypothesis space $\mathcal{H}$ when for any $l \in \mathcal{L}$, any $h \in \mathcal{H}$ and any two samples $\mathcal{S}$ and $\mathcal{S}'$ the empirical $l$-risk is such that
$$
R_{\calS, l}(h) - R_{\calS', f}(h) \text{~does not depend on~} z \Leftrightarrow T(\mathcal{S}) = T(\mathcal{S}')\:\:.
$$
\end{definition}
The definition is motivated by the one in Statistics, taking log-odd ratios \citep{pnrcAN}. We aim to establish sufficiency of mean operators for a large set of losses.  The next theorem is our main contribution. 

\begin{theorem}[Factorization]\label{th:factorization}
Let $\mathcal{H}$ be the space of linear hypotheses. Assume that a loss $l$ is such that $l_o(x) \defeq (l(x) - l(-x)) / 2$ is linear. Then, for any sample $\calS$ and hypothesis $h \in \mathcal{H}$ the empirical $l$-risk can be written as
\begin{align*}
R_{\calS, l}(h) = \frac{1}{2} R_{\calS_{2x}, l}(h) + l_o(h(\meanop_{\calS})) \:\:,
\end{align*}
 where $\calS_{2x} \defeq \{ (\bm{x}_i, \sigma), i \in [m], \forall \sigma \in \mathcal{Y} \}$.
\end{theorem}
\begin{proof}
We write $R_{\calS, l}(h) =  \expect_{\calS} [ l(yh(\bm{x}))$ ] as 
\begin{align}
 & \frac{1}{2}  \expect_{\calS} \Big[ l(y h(\bm{x})) + l(-y h(\bm{x}))
 + l(y h(\bm{x})) - l(-y h(\bm{x})) \Big] \nonumber \\
&=  \frac{1}{2} \expect_{\calS} \left[ \sum\nolimits_{\sigma\in \mathcal{Y}} l(\sigma h(\bm{x})) \right] + \expect_{\calS} \Big[ l_o(y h(\bm{x})) \Big] \nonumber \\
&=  \frac{1}{2} \expect_{\calS_{2x}} \Big[ l(\sigma h(\bm{x})) \Big] + \expect_{\calS} \Big[ l_o(h(y \bm{x})) \Big] \label{eq:fact1} \:\:.
\end{align}
 Step \ref{eq:fact1} is due to the definition of $\calS_{2x}$ and linearity of $h$. The Theorem follows by linearity of $l_o$ and expectation.
\end{proof}
 Factorization splits $l$-risk in two parts. A first term is the $l$-risk computed \emph{on the same loss} on the ``doubled sample" $\calS_{2x}$ that contains each observation twice, labelled with opposite signs, and hence it is label free. A second term is a loss $l_o$ of $h$ applied to the mean operator $\meanop_{\calS}$, which aggregates all sample labels. Also observe that $l_o$ is by construction an odd function, \emph{i.e.} symmetric \emph{w.r.t.} the origin. We call the losses satisfying the Theorem linear-odd losses. 
\begin{definition}
A loss $l$ is $a$-linear-odd ($a$-\lol) when $l_o(x) = (l(x) - l(-x))/2 =ax$, for any $a \in \R$.
\end{definition}
Notice how this does not exclude losses that are not smooth, convex, or proper \citep{rwCB}. The definition puts in a formal shape the intuition behind \citep{pnsCF} for \textsc{pu} --although unrelated to factorization. From now on, we also consider $\mathcal{H}$ as the space linear hypotheses $h(\cdot) = \langle \bm{\theta}, \cdot \rangle $. (Theorem \ref{th:factorization} applies beyond \lol s and linear models; see Section \ref{sec:concl}.) As a consequence of Theorem \ref{th:factorization}, $\meanop$ is sufficient for all labels.

\begin{table}[t!]
\centering
$$
\begin{array}{|l|c|c|} \hline
\mbox{}  & \mbox{loss $l$}  & \mbox{odd term $l_o$} \\ \toprule \hline 
\mbox{\lol}  &  l(x)  & -ax \\ \hline
\mbox{$\rho$-loss}  & \rho |x|-\rho x+1 & -\rho x \:\:(\rho \geq 0) \\ \hline
\mbox{unhinged}  &  1-x & -x \\ \hline
\mbox{perceptron} & \max(0, -x) & -x \\ \hline
\mbox{2-hinge} & \max(-x, 1/2 \max (0,  1-x)) & -x \\ \hline
\mbox{\spl}  &  a_l+ l^\star(-x) / b_l & - x / (2 b_l)\\
\mbox{logistic}  & \log (1+e^{-x}) & - x /2\\ 
\mbox{square} & (1-x)^2 & -2x \\
\mbox{Matsushita} & \sqrt{1+x^2}-x & -x \\ \bottomrule
\end{array}
$$
\caption{Factorization of linear-odd losses: \spl~(including logistic, square and Matsushita) \citep{nnBD}, double ``2"-hinge and perceptron \citep{pnsCF}, unhinged \citep{rmwLW}. For $\rho$-loss see the text.
\label{table:1}}
%\vspace{-10px}
\end{table}

\begin{corollary}\label{th:suff}
The mean operator $\meanop$ is a sufficient statistic for the label $y$ with regard to \lol s and $\mathcal{H}$.
\end{corollary}
(Proof in \ref{proof:suff}) The corollary is at the core of the applications in the paper: the single vector $\bm{\mu} \in \R^d$ summarizes all information concerning the linear relationship between $y$ and  $\bm{x}$ for losses that are \lol~(see also Section \ref{sec:concl}). Many known losses belong to this class; see Table~\ref{table:1}. For logistic loss it holds that (Figure \ref{fig:1-logistic}):
\begin{align*}
l_o(x) = \frac{1}{2} \log \frac{1+e^{-x}}{1+e^{x}} =
\frac{1}{2} \log \frac{e^{-\frac{x}{2}} (e^{\frac{x}{2}}+e^{-\frac{x}{2}})}{e^{\frac{x}{2}} (e^{-\frac{x}{2}}+e^{\frac{x}{2}}) }= -\frac{x}{2}
\end{align*}
 This ``symmetrization" is known in the literature \citep{jjBP, gwlzRM}. Another case of \lol~is unhinged loss $l(x) = 1-x$ \citep{rmwLW} --while standard hinge loss does not factor in a linear term.

The Factorization Theorem \ref{th:factorization} generalizes  \citep[Lemma 1]{pnrcAN} that works for \emph{symmetric proper losses} (\spl s), \emph {e.g.} logistic, square and Matsushita losses. Given a permissible generator $l$ \citep{kmOT, nnBD}, \textit{i.e.} $\mathrm{dom}(l) \supseteq [0,1]$, $l$ is strongly convex, differentiable and symmetric with respect to $1/2$, \spl s are defined as $l(x) = a_l+ l^\star(-x) / b_l $, where $l^\star$ is the convex conjugate of $l$. Then, since $l^\star(-x)=l^\star(x)-x$:
\begin{align*}
l_o(x) & = \frac{1}{2} \left( a_l+\frac{l^\star(-x)}{b_l}-a_l-\frac{l^\star(x)}{b_l} \right) = -\frac{x}{2b_l} \:\:.
\end{align*}
A similar result appears in \citep[Theorem 11]{msTD}. A natural question is whether the classes \spl~and \lol~are equivalent. We answer in the negative.

\begin{lemma}\label{th:1-to-1even}
The exhaustive class of linear-odd losses is in 1-to-1 mapping with a proper subclass of even functions, \emph{i.e.} $l_e(x) - ax$, with $l_e$ any even function.
\end{lemma}
(Proof in \ref{proof:1-to-1even}) Interestingly, the proposition also let us engineer losses that always factor: choose any even function $l_e$ with desired properties --it need not be convex nor smooth. The loss is then $l(x) = l_e(x) - ax$, with $a$ to be chosen. For example, let $l_e(x) = \rho |x| + 1$, with $\rho > 0$. $l(x) = l_e(x) - \rho x$ is a \lol; furthermore, $l$ upperbounds $01$ loss and intercepts it in $l(0)=1$. Non-convex $l$ can be constructed similarly. Yet, not all non-differentiable losses can be crafted this way since they are not \lol s.  We provide in \ref{proof:upper} sufficient and necessary conditions to bound losses of interest, including hinge and Huber loss, by \lol s.

From the optimization viewpoint, we may be interested in keeping properties of $l$ after factorization. The good news is that we are dealing with the same $l$ plus a linear term. Thus, if the property of interest is closed under summation with linear functions, then it will hold true. An example is convexity: if $l$ is \lol~and convex, so is the factored loss.

The next Theorem sheds new light on generalization bounds on Rademacher complexity with linear hypotheses.
\begin{theorem}\label{th:generalization}
Assume $l$ is  $a$-\lol~and $L$-Lipschitz. Suppose $\R^d \supseteq \mathcal{X} = \{ \bm{x}: \| \bm{x} \|_2 \leq X < \infty \}$ and $\mathcal{H} = \{ \bm{\theta}: \|\bm{\theta}\|_2 \leq B < \infty \}$. Let $c(X, B) \defeq \max_{y \in \mathcal{Y}} l(y XB)$ and $\hat{\bm{\theta}} \defeq \argmin_{\bm{\theta} \in \mathcal{H}} R_{\calS, l}(\bm{\theta})$. Then for any $\delta > 0$, with probability at least $1 - \delta$:
\begin{align*}
&R_{\calD, l}(\hat{\bm{\theta}}) - \inf_{\bm{\theta} \in \mathcal{H}} R_{\calD, l}(\bm{\theta}) \leq \left( \frac{\sqrt{2} + 1}{4} \right) \cdot \frac{XBL}{\sqrt{m}} + \\
&\frac{c(X, B) L}{2} \cdot \sqrt{\frac{1}{m}\log\left(\frac{1}{\delta}\right)} + 2|a|B \cdot \| \meanop_{\calD} - \meanop_{\calS}\|_2\:\:,
\end{align*}
or more explicity 
\begin{align*}
&R_{\calD, l}(\hat{\bm{\theta}}) - \inf_{\bm{\theta} \in \mathcal{H}}  R_{\calD, l}(\bm{\theta}) \leq \left( \frac{\sqrt{2} + 1}{4} \right) \cdot \frac{XBL}{\sqrt{m}} + \\
&\left( \frac{c(X, B) L}{2} + 2|a|XB \sqrt{d \log{d}} \right) \sqrt{\frac{1}{m}\log\left(\frac{2}{\delta}\right)}\:\:.
\end{align*}
\end{theorem}

(Proof in \ref{proof:generalization}) The term $\frac{\sqrt{2} + 1}{4} \cdot \frac{XBL}{\sqrt{m}}$ is derived by an improved upperbound to the Rademacher complexity of $\mathcal{H}$ computed on $\calS_{2x}$ (\ref{proof:generalization}, Lemma \ref{lemma1}); we call it in short the \emph{complexity} term. The former expression displays the contribution of the non-linear part of the loss, keeping aside what is missing: a deviation of the empirical mean operator from its population mean. When $\meanop_{\calS}$ is not known because of partial label knowledge, the choice of any estimator would affect the bound only through that norm discrepancy. The second expression highlights the interplay of the two loss components. $c(X, B)$ is the only non-linear term, which may well be predominant in the bound for fast-growing losses, \emph{e.g.} strongly convex. Moreover, we confirm that the linear-odd part does not change the complexity term and only affects the usual statistical penalty by a linear factor. A last important remark comes from comparing the bound with the one due to \citep[Corollary 4]{kstOT}. Our factor in front of the complexity term is $(\sqrt{2} + 1)/4 \approx 0.6$ instead of $2$, that is three times smaller. A similar statement may be derived for \textsc{rkhs} on top of \citep{bmRA, asUD}.

\section{Weakly supervised learning}\label{sec:weakly}
\begin{algorithm}[t]
\caption{$\meanop$\textsc{sgd}}\label{algo:meta}
\begin{algorithmic}
\STATE \textbf{Input}: \fcolorbox{gray!35}{gray!35}{$\calS_{2x}, \meanop$}, $l \in \lol$; $\lambda > 0$; $T > 0$\;
\STATE $m' \leftarrow | \calS_{2x} |$\;
\STATE $\bm{\theta}^0 \leftarrow \bm{0}$\;
\STATE For any $t = 1, \dots, T$:
\STATE \qquad Pick $i^t \in [m']$ uniformly at random\;
\STATE \qquad $\eta^t \leftarrow (1 + \lambda t)^{-1}$\;
\STATE \qquad Pick any $\bm{v} \in \partial l(y_i \langle \bm{\theta}^{t}, \bm{x}_i \rangle)$\;
\STATE \qquad $\bm{\theta}^{t+1} \leftarrow (1 - \eta^t \lambda) \bm{\theta}^{t} - \eta^t (\bm{v}$ \fcolorbox{gray!35}{gray!35}{$+ a \meanop /2$}~)\; 
\STATE \qquad $\bm{\theta}^{t+1} \leftarrow \min \left\{ \bm{\theta}^{t+1}, \bm{\theta}^{t+1} \sqrt{\lambda^{-1}} / \| \bm{\theta}^{t+1} \|_2) \right\} $\;
\STATE \textbf{Output}: $\bm{\theta}^{t+1}$\;
\end{algorithmic}
\end{algorithm}

In the next two Sections we discuss applications to \textsc{wsl}. Recall that in this scenario we learn on $\tilde{\calS}$ with partially observable labels, but aim to generalize to $\calD$. Assume to know an algorithm that can only learn on $\calS$. By sufficiency (Corollary \ref{th:suff}), a principled two-step approach to use it on $\tilde{\calS}$ is: (1) estimate $\meanop$ from $\tilde{\calS}$; (2) run the algorithm with the \lol~computed on the estimated $\meanop$. This direction was explored by work on \textsc{llp} \citep[logistic loss]{qsclEL} and \citep[\spl]{pnrcAN} and in the setting of noisy labels \citep[unhinged loss]{rmwLW} and \citep[logistic loss]{gwlzRM}. The approach contrasts with \emph{ad-hoc} losses and optimization procedures, often trying to recover the latent labels by coordinate descent and EM \citep{jbAC}. Instead, the only difficulty here is to come up with a well-behaved estimator of $\meanop$ --a statistic independent on both $h$ and $l$. Thereom \ref{th:generalization} then assures bounded $l$-risk and, in turn, true risk. On stricter conditions on $l$ \citep[Section 4]{asUD} and \citep[Theorem 6]{pnrcAN} hold finite-sample guarantees.

Algorithm \ref{algo:meta}, $\meanop$\textsc{sgd}, adapts \textsc{sgd} for weak supervision. For the sake of presentation, we work on a simple version of \textsc{sgd} based on subgradient descent with $L_2$ regularization, inspired by PEGASO \citep{sssscPP}. Given $\meanop$ changes are trivial: (i) construct $\calS_{2x}$ from $\tilde{\calS}$ and (ii) sum $-a\meanop / 2$ to the subgradients of each example of $\calS_{2x}$. In Section \ref{sec:concl} upgrades proximal algorithms with the same minimal-effort strategy. The next Section shows an estimator of $\meanop$ in the case of noisy labels and specializes $\meanop$\textsc{sgd}. We also analyze the effect of noise through the lenses of Theorem \ref{th:generalization} and discuss a non-standard notion of noise robustness.

\section{Asymmetric label noise}\label{sec:noisy}

In learning with noisy labels, $\tilde{\calS}$ is a sequence of examples drawn from a distribution $\tilde{\calD}$, which samples from $\calD$ and flips labels at random. Each example $(\bm{x}_i, \tilde{y}_i)$ is $(\bm{x}_i, -y_i)$ with probability at most $1/2$ or it is $(\bm{x}_i, y_i)$ otherwise. The \emph{noise rates} are label dependent\footnote{While being independent on the observation.} by $(p_+, p_-) \in [0, 1/2)^2$ respectively for positive and negative examples, that is, \emph{asymmetric} label noise (\textsc{aln}) \citep{ndrtLW}.

Our first result builds on \citep[Lemma 1]{ndrtLW} that provides a recipe for unbiased estimators of losses. Thanks to the Factorization Theorem \ref{th:factorization}, instead of estimating the whole $l$ we act on the sufficient statistic:
\begin{align}
\hat{\meanop}_{\calS} \defeq \expect_{\calS} \left[ \frac{y - (p_- - p_+)}{1 - p_- - p_+} \bm{x} \right] \:\:. \label{eq:unbiasedmu}
\end{align}
 The estimator is unbiased, that is, its expectation over the noise distribution $\tilde{\calD}$ is the population mean operator: $\hat{\bm{\mu}}_{\tilde{\calD}} = \bm{\mu}_{\calD}$. Denote then the risk computed on the estimator as $\hat{R}_{\calS, l}(\bm{\theta}) \defeq \frac{1}{2} R_{\calS_{2x}, l}(\bm{\theta}) + a \langle \bm{\theta}, \hat{\meanop}_{\calS} \rangle$. Unbiasedness transfers to $l$-risk: $\hat{R}_{\tilde{\calD}, l}(\bm{\theta}) = R_{\calD, l}(\bm{\theta}),~\forall \bm{\theta}$ (Proofs in \ref{proof:unbiased}). We have therefore obtained a good candidate as input for any algorithm implementing our 2-step approach, like $\meanop$\textsc{sgd}. However, in the context of the literature, there is more. On one hand, the estimators of \citep{ndrtLW} may not be convex even when $l$ is so, but this is never the case with \lol s; in fact, $l(x) - l(-x) = 2ax$ may be seen as alternative sufficient condition to \citep[Lemma 4]{ndrtLW} for convexity, without asking $l$ differentiable, for the same reason in \citep{pnsCF}. On the other hand, we generalize the approach of \citep{rmwLW} to losses beyond unhinged and to asymmetric noise. We now prove that \emph{any} algorithm minimizing \lol s that uses the estimator in Equation \ref{eq:unbiasedmu} has a non-trivial generalization bound. We further assume that $l$ is Lipschitz.

\begin{theorem}\label{th:noisy1}
Consider the setting of Theorem \ref{th:generalization}, except that here $\hat{\bm{\theta}} = \argmin_{\bm{\theta} \in \mathcal{H}} \hat{R}_{\tilde{\calS}, l}(\bm{\theta})$.  Then for any $\delta > 0$, with probability at least $1 - \delta$:
\begin{align*}
&R_{\calD, l}(\hat{\bm{\theta}}) - \inf_{\bm{\theta} \in \mathcal{H}} R_{\calD, l}(\bm{\theta}) \leq \left( \frac{\sqrt{2} + 1}{4} \right) \cdot \frac{XBL}{\sqrt{m}} + \\ 
&\left( \frac{c(X, B) L}{2} + \frac{2|a|XB}{1 - p_- - p_+} \sqrt{d \log{d}} \right) \sqrt{\frac{1}{m}\log\left(\frac{2}{\delta}\right)}\:\:.
\end{align*}
\end{theorem}

(Proof in \ref{proof:generalization-noisy}) Again, the complexity term is tighter than prior work. \citep[Theorem 3]{ndrtLW} proves a factor of $2L / (1 - p_- - p_+)$ that may even be unbounded due to noise, while our estimate shows a constant of about $ 0.6 < 2$ and it is noise free. In fact, \lol s are such that noise affects only the linear component of the bound, as a direct effect of factorization. Although we are not aware of any other such results, this is intuitive: Rademacher complexity is computed regardless of  sample labels and therefore is unchanged by label noise. Furthermore, depending on the loss, the effect of (limited) noise on generalization may be also be negligible since $c(X, B)$ could be very large for losses like strongly convex. This last remark fits well with the property of robustness that we are about to investigate.

\subsection{Every \lol~is approximately noise-robust}

The next result comes in pair with Theorem \ref{th:noisy1}: it holds regardless of algorithm and (linear-odd) loss of choice. In particular, we demonstrate that every learner enjoys a distribution-dependent property of robustness against asymmetric label noise. No estimate of $\meanop$ is involved and hence the theorem applies to any na{\"i}ve supervised learner on $\tilde{\calS}$. We first refine the notion of robustness of \citep{gmsMR, rmwLW} in a weaker sense.

\begin{definition}
Let $(\bm{\theta}^\star, \tilde{\bm{\theta}}^\star)$ respectively be the minimizers of $(R_{\calD, l}(\bm{\theta}), R_{\tilde{\calD}, l}(\bm{\theta}))$ in $\mathcal{H}$. $l$ is said $\epsilon$-\textsc{aln} robust if for any $\calD, \tilde{\calD}$, $R_{\tilde{\calD}, l}(\bm{\theta}^\star) - R_{\tilde{\calD}, l}(\tilde{\bm{\theta}}^\star) \leq \epsilon$.
\end{definition}

The distance of the two minimizers is measured by empirical $l$-risk under expected label noise. $0$-\textsc{aln} robust losses are also \textsc{aln} robust: in fact if  $R_{\tilde{\calD}, l}(\bm{\theta}^\star) = R_{\tilde{\calD},l}(\tilde{\bm{\theta}}^\star)$ then $\bm{\theta}^\star \in \argmin_{\bm{\theta}} R_{\tilde{\calD}, l}(\bm{\theta})$. And hence if $R_{\tilde{\calD}, l}(\bm{\theta})$ has a unique global minimum, that will be $\bm{\theta}^\star$. More generally

\begin{theorem}\label{th:noisy2}
Assume $\{ \bm{\theta} \in \mathcal{H}: ||\bm{\theta}||_2 \leq B \}$. Then every $a$-\lol~is $\epsilon$-\textsc{aln}. That is
\begin{align*}
R_{\tilde{\calD}, l}(\bm{\theta}^\star) - R_{\tilde{\calD}, l}(\tilde{\bm{\theta}}^\star)
\leq 4 |a| B  \max\{p_+, p_-\}  \|\meanop_{\calD}\|_2
\end{align*}
Moreover: (1) If $\| \meanop_{\calD} \|_2 = 0$ for $\calD$ then every \lol~is \textsc{aln} for any $\tilde{\calD}$. (2) Suppose that $l$ is also once differentiable and $\gamma$-strongly convex. Then $\| \bm{\theta}^\star - \tilde{\bm{\theta}}^\star \|^2_2 \leq 2 \epsilon / \gamma$ .
\end{theorem}

\begin{figure}[t]
\centering
\subfigure{\label{fig:2-phi} \includegraphics[width=.30\textwidth]{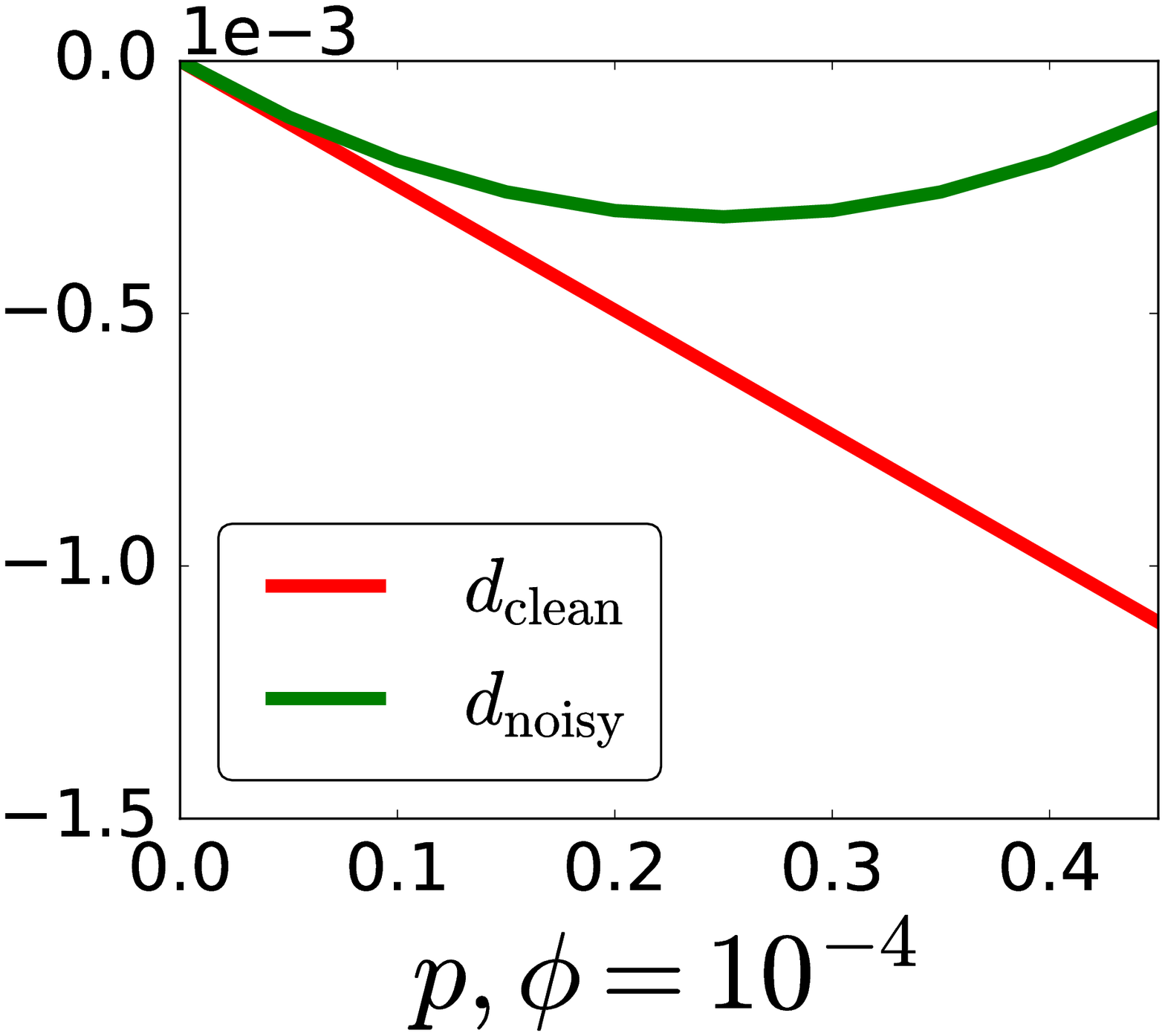}}
\subfigure{\label{fig:2-p} \includegraphics[width=.30\textwidth]{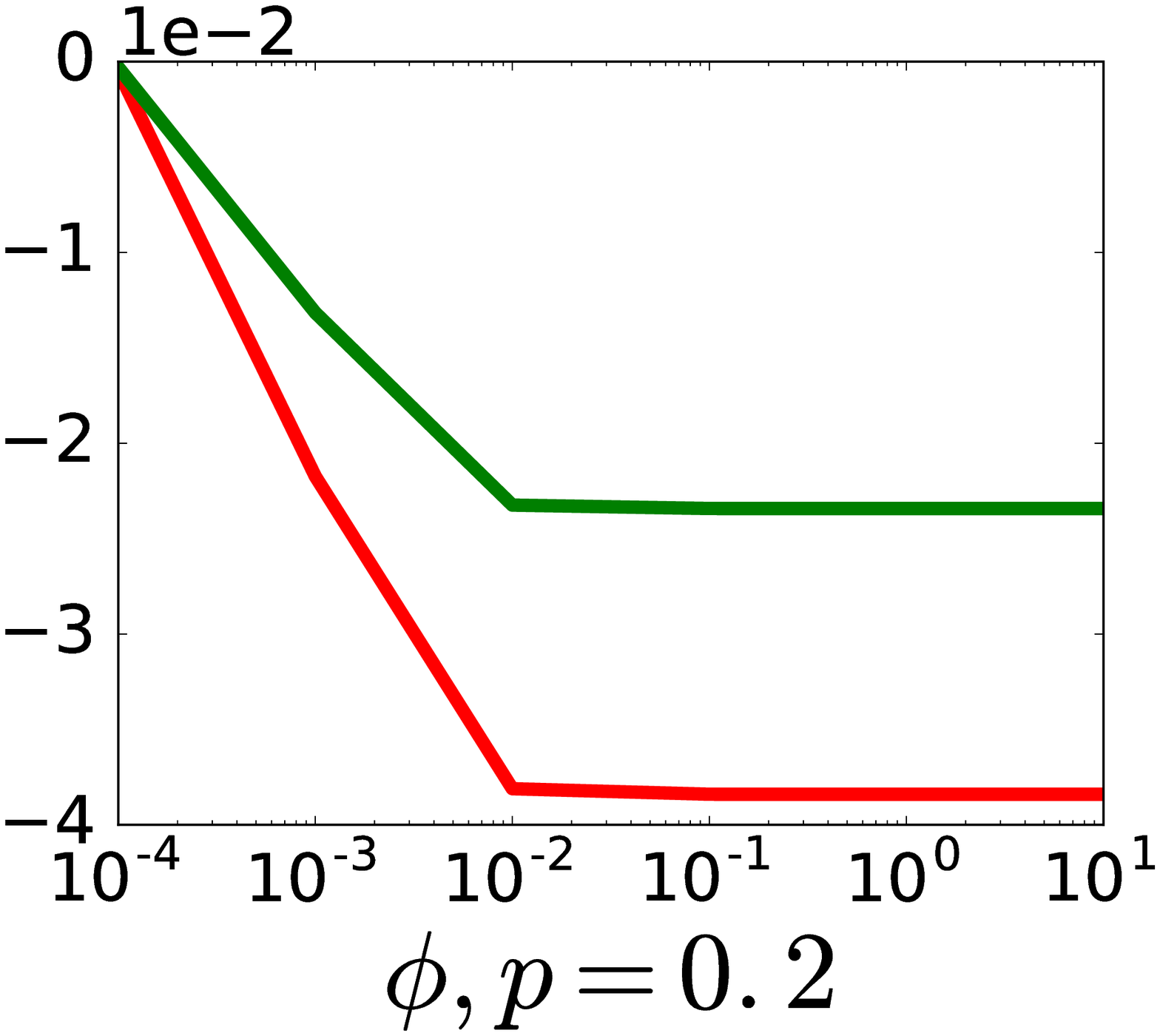}}
\vspace{-10px}
\caption{Behavior of Theorem \ref{th:noisy2} on synthetic data.}\label{fig:exp1}
\vspace{-10px}
\end{figure}

(Proof in \ref{proof:noisy1}) Unlike Theorem \ref{th:noisy1}, this bound holds in expectation over the noisy risk $R_{\tilde{D}, l}$. Its shape depends on the population mean operator, a \emph{distribution-dependent} quantity. There are two immediate corollaries. When $\|\meanop_{\calD}\|_2 = 0$, we obtain optimality for all \lol s. The second corollary goes further, limiting the minimizers' distance when losses are differentiable and strongly convex. But even more generally, under proper compactness assumptions on the domain of $l$, Theorem \ref{th:noisy2} tells us much more: in the case $R_{\tilde{\calD}, l}(\bm{\theta})$ has a unique global minimizer, the smaller $\| \meanop_{\calD} \|_2$, the closer the minimizer \emph{on noisy data} $\tilde{\bm{\theta}}^\star$ will be to the minimizer \emph{on clean data} $\bm{\theta}^\star$. Therefore, assuming to know an efficient algorithm that computes a model not far from the global optimum $\tilde{\bm{\theta}}^\star$, that will be not far from $\bm{\theta}^\star$ either. This is true in spite of the presence of local minima and/or saddle points.

\citep{lsRC} proves that no convex potential\footnote{A convex potential is a loss $l \in C^1$, convex, such that $l(0)<0$ and $l(x)\rightarrow 0$ for $x \rightarrow \infty$. Many convex potentials are \lol s but not all, but there is no inclusion. An example is $e^{-x}$.} is noise tolerant, that is, $0$-\textsc{ALN} robust. This is not a contradiction. To show the negative statement, the authors craft a case of $\calD$ breaking any of those losses. And in fact that choice of $\calD$ does not meet optimality in our bound, because $\|\meanop_{\calD}\|_2 = \frac{1}{4}(18\gamma^2 + 6\gamma + 1) > 0$, with $\gamma \in (0, 1/6)$. In contrast, we show that every element of the broad class of \lol s is approximately robust, as opposed to a \emph{worst-case} statement. Finally, compare our $\epsilon$-robustness to the one of \citep{gmsMR}: $R_{\calD, l}(\tilde{\bm{\theta}}^*) \leq (1 - 2\max(p_-,p_+))^{-1} R_{\calD, l}(\bm{\theta}^*) $. Such bound, while relating the (non-noisy) $l$-risks, is not data-dependent and may be not much informative for high noise rates.

\subsection{Experiments}

\begin{algorithm}[t]
\caption{$\meanop$\textsc{sgd} applied on noisy labels}\label{algo:noisy}
\begin{algorithmic}
\STATE \textbf{Input}: $\tilde{\calS}, l \in \lol$; $\lambda > 0$; $T > 0$\;
\STATE $\calS_{2x} \defeq \{ (\bm{x}_i, \sigma), i \in [m], \forall \sigma \in \mathcal{Y} \}$
\STATE $\hat{\meanop}_{\tilde{\calS}} \leftarrow $ Equation \ref{eq:unbiasedmu}
\STATE $\bm{\theta} \leftarrow \meanop$-\textsc{sgd}$(\calS_{2x}, \hat{\meanop}_{\tilde{\calS}}, \lambda, T)$\;
\STATE \textbf{Output}: $\bm{\theta}$
\end{algorithmic}
\end{algorithm}

We analyze experimentally the theory so far developed. From now on, we assume to know $p_+$ and $p_-$ at learning time. In principle they can be tuned as hyper-parameters \cite{ndrtLW}, at least for small $|\mathcal{Y}|$ \citep{sfLF}. While being out of scope, practical noise estimators are studied \citep{bkLN, ltCW, mvrowLF, sAR}.

\begin{figure*}[t]
\centering
\subfigure[$p\|\mu_{\calD}\|_2$ \emph{vs.} $d_{\mathrm{clean}}$]{\label{fig:3-a} \includegraphics[width=.39\textwidth]{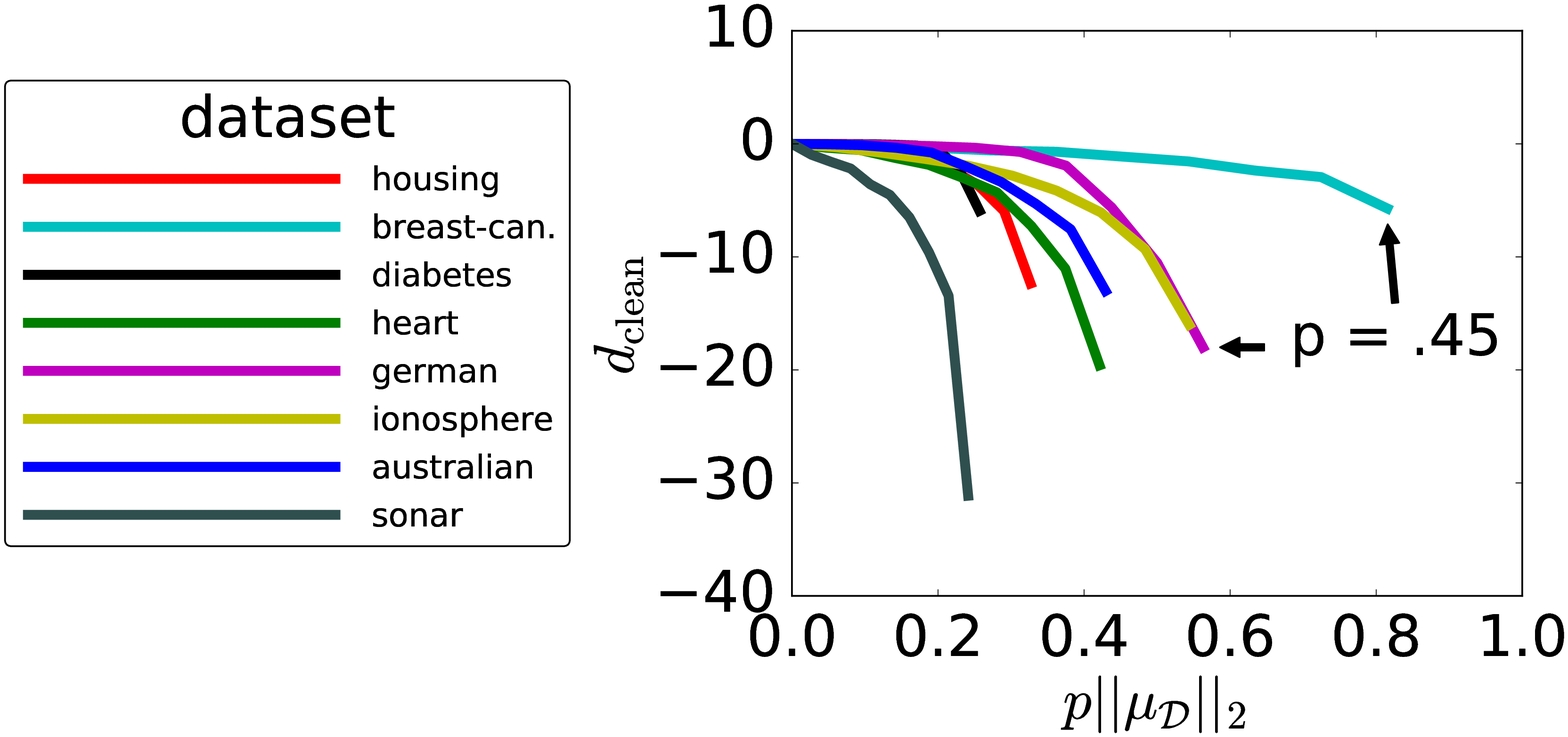}}
\subfigure[$p\|\mu_{\calD}\|_2$ \emph{vs.} $R_{\mathcal{D}, 01}$]{\label{fig:3-b} \includegraphics[width=.24\textwidth]{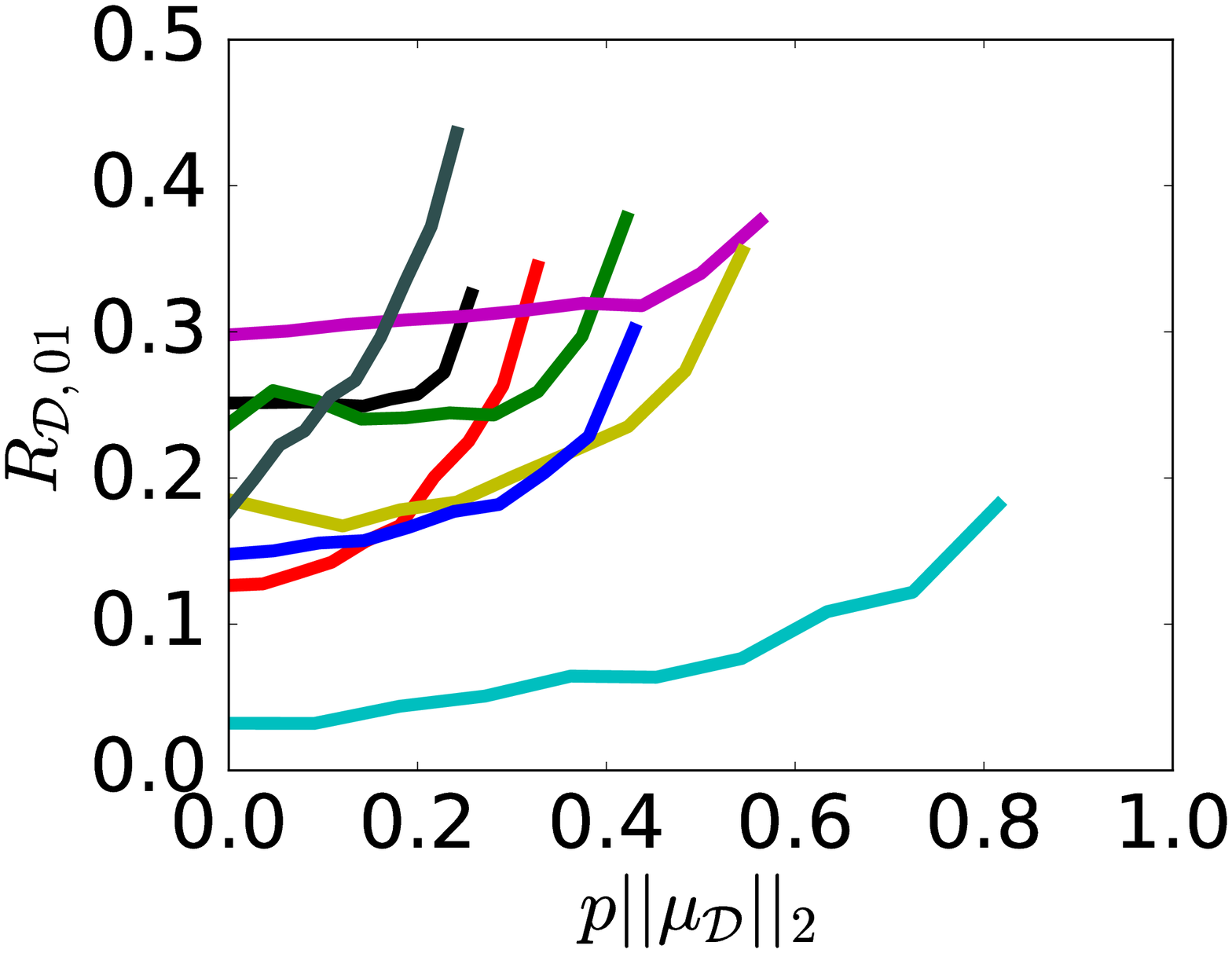}} 
\subfigure[$\|\mu_{\tilde{\calD}}\|_2$ \emph{vs.} $d_{\mathrm{clean}}$]{\label{fig:3-c} \includegraphics[width=.245\textwidth]{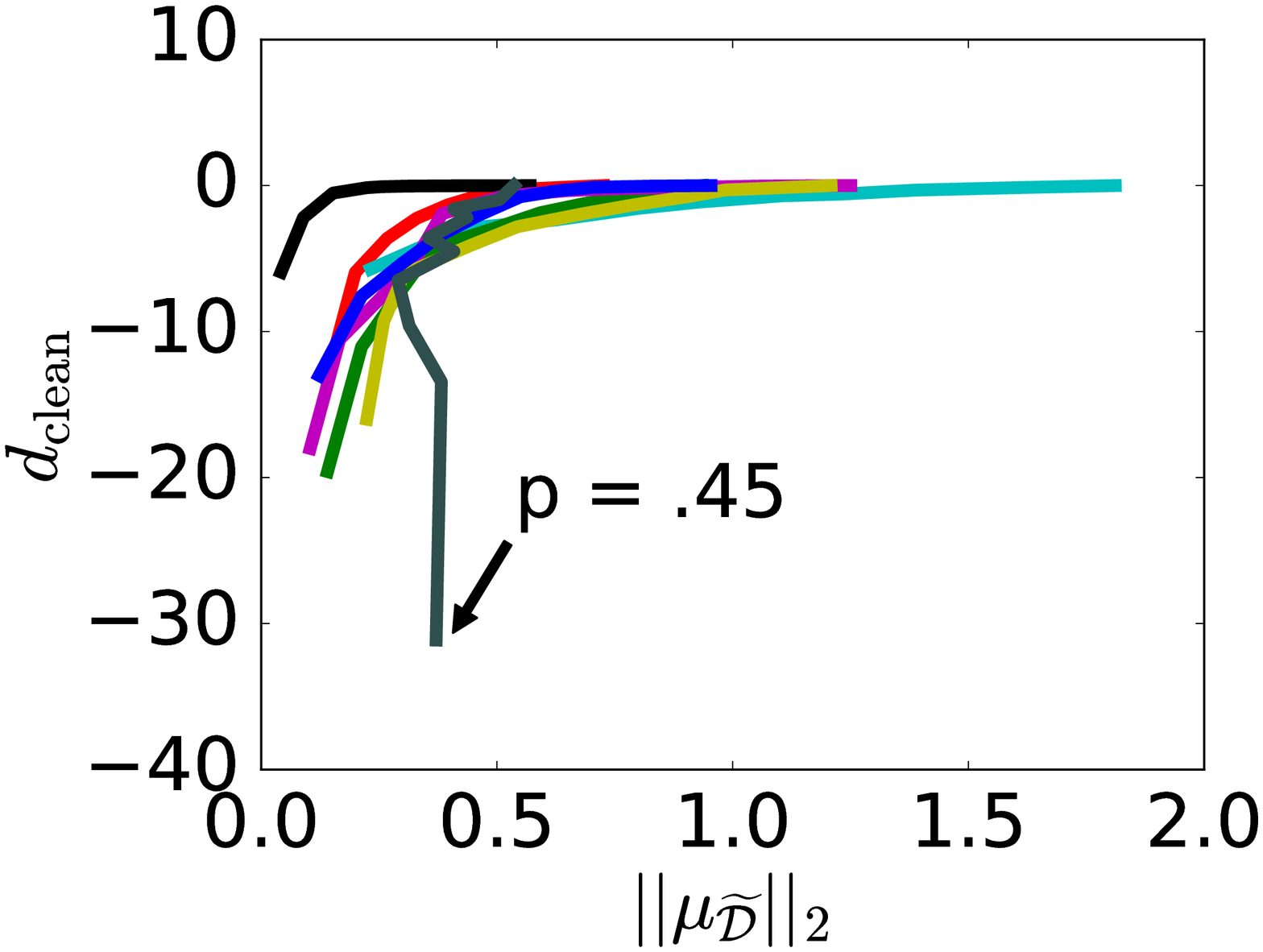}}
\vspace{-10px}
\caption{How mean operator and noise rate condition risks. $d_{\mathrm{clean}} \defeq R_{\calD, l}(\bm{\theta}^\star) - R_{\calD, l}(\tilde{\bm{\theta}}^\star)$.} \label{fig:exp2}
\vspace{-10px}
\end{figure*}

We begin by building a toy planar dataset to probe the behavior of Theorem \ref{th:noisy2}. It is made of four observations: $(0, 1)$ and  $(\phi / 3, 1/3)$ three times, with the first example the only negative, repeated 5 times. We consider this the distribution $\calD$ so as to calculate $\| \meanop_{\calD} \|_2 = \phi^2 / 4$. We fix $p_+, p_- = 0.2 = p$ and control $\phi$ to measure the discrepancy $d_{\mathrm{noisy}} \defeq R_{\tilde{\calD}, l}(\bm{\theta}^\star) - R_{\tilde{\calD}, l}(\tilde{\bm{\theta}}^\star)$, its counterpart $d_{\mathrm{clean}}$ computed on $\calD$, and how the two minimizers ``differ in sign" by $d_{\mathrm{models}} \defeq \langle \bm{\theta}^\star, \tilde{\bm{\theta}}^\star \rangle / \| \bm{\theta}^\star \|_2 \| \tilde{\bm{\theta}}^\star \|_2 $. The same simulation is run varying the noise rates with constant $\phi = 10^{-4}$. We learn with $\lambda = 10^{-6}$ by standard square loss. Results are in Figure \ref{fig:exp1}. The closer the parameters to $0$, the smaller $d_{\mathrm{clean}} -d_{\mathrm{noisy}}$, while they are equal when each parameter is individually $0$. $d_{\mathrm{models}}$ is negligible, which is good news for the $01$-risk on sight.

Algorithm \ref{algo:noisy} learns with noisy labels on the estimator of Equation \ref{eq:unbiasedmu} and by calling the black box of $\meanop$\textsc{sgd}. The next results are based on UCI datasets \citep{blUM}. We learn with logistic loss, without model's intercept and set $\lambda = 10^{-6}$ and $T = 4 \cdot 2m$ (4 epochs). We measure $d_{\mathrm{clean}}$ and $R_{\calD, 01}$, injecting symmetric label noise $p \in [0, 0.45)$ and averaging over 25 runs. Again, we consider \emph{the whole distribution} so as to play with the ingredients of Theorem \ref{th:noisy2}. Figure \ref{fig:3-a} confirms how the combined effect of $p \| \meanop_{\calD} \|_2$ can explain most variation of $d_{\mathrm{clean}}$. While this is not strictly implied by Theorem \ref{th:noisy2} that only involves $d_{\mathrm{noisy}}$, the observed behavior is expected. A similar picture is given in Figure \ref{fig:3-b} which displays the true risk $R_{\calD, 01}$ computed on the minimizer $\tilde{\bm{\theta}}^\star$ of $\tilde{\calS}$. From \ref{fig:3-a} and \ref{fig:3-b} we also deduce that large $\| \meanop_{\calD} \|_2$ is a good proxy for generalization with linear classifiers; see the relative difference between points at the same level of noise. Finally, we have also monitored $\meanop_{\tilde{\calD}}$. Figure \ref{fig:3-c} shows that large $\| \meanop_{\tilde{\calD}} \|_2$ indicates small $d_{\mathrm{clean}}$ as well. This remark can be useful in practice, when the norm can be accurately estimated from $\tilde{\calS}$, as opposite to $p$ and $\meanop_{\calD}$, and used to anticipate the effect of noise on the task at hand.

We conclude with a systematic study of hold-out error of $\meanop$\textsc{sgd}. The same datasets are now split in 1/5 test and 4/5 training sets once at random. In contrast with the previous experimental setting we perform cross-validation of $\lambda \in 10^{\{-3, \dots, +3 \}}$ on 5-folds in the training set. We compare with vanilla \textsc{sgd} run on corrupted sample $\tilde{\calS}$ and measure the gain from estimating $\hat{\meanop}_{\tilde{\calS}}$. The other parameters $l, T, \lambda$ are the same for both algorithms; the learning rate $\eta$ is untouched from \citep{sssscPP} and not tuned for $\meanop$\textsc{sgd}. The only differences are in input and gradient update. Table \ref{table:big} reports test error for \textsc{sgd} and its difference with $\meanop$\textsc{sgd}, for a range of values of $(p_-, p_+)$. $\meanop$\textsc{sgd} beats systematically \textsc{sgd} with large noise rates, and still performs in pair with its competitor under low or null noise. Interestingly, in the presence of very intense noise $p_+ \approx .5$, while the contender is often doomed to random guess, $\meanop$\textsc{sgd} still learns sensible models and improves up to $55\%$ relatively to the error of \textsc{sgd}.

\begin{table*}
\centering
\tiny
\begin{tabular}{@{}l@{}lr|rr|rr|rr|rr|rr|rrr}
\toprule
$(p_-, p_+) \rightarrow$                  & \multicolumn{2}{c}{$(.00,.00)$} & \multicolumn{2}{c}{$(.20,.00)$} & \multicolumn{2}{c}{$(.20,.10)$} & \multicolumn{2}{c}{$(.20,.20)$} & \multicolumn{2}{c}{$(.20,.30)$} & \multicolumn{2}{c}{$(.20,.40)$} & \multicolumn{2}{c}{$(.20,.49)$} \\
\emph{dataset}                  & \textsc{sgd} & $\meanop$\textsc{sgd} & \textsc{sgd} & $\meanop$\textsc{sgd} & \textsc{sgd} & $\meanop$\textsc{sgd} & \textsc{sgd} & $\meanop$\textsc{sgd} & \textsc{sgd} & $\meanop$\textsc{sgd} & \textsc{sgd} & $\meanop$\textsc{sgd} & \textsc{sgd} & $\meanop$\textsc{sgd} \\
\midrule
 australian &      0.13 &          $+.01$ &      0.15 &          $\bm{-.01}$ &      0.14 &          $\pm.00$ &      0.14 &          $+.01$ &      0.16 &          $\bm{-.01}$ &      0.26 &          $\bm{-.09}$ &      0.45 &          $\bm{-.25}$ \\
       breast-can. &      0.02 &          $+.01$ &      0.03 &          $\pm.00$  &      0.03 &          $\pm.00$  &      0.03 &          $\pm.00$  &      0.05 &          $\bm{-.01}$  &      0.11 &          $\bm{-.06}$ &      0.17 &          $\bm{-.08}$ \\
       diabetes &      0.28 &          $\bm{-.03}$ &      0.29 &          $\bm{-.03}$ &      0.29 &          $\bm{-.03}$ &      0.27 &          $\bm{-.02}$ &      0.28 &          $\bm{-.02}$ &      0.39 &          $\bm{-.13}$ &      0.59 &          $\bm{-.22}$ \\
       german &      0.27 &          $\bm{-.02}$ &      0.26 &          $\pm.00$ &      0.27 &          $\bm{-.02}$ &      0.29 &          $\bm{-.02}$ &      0.31 &          $\bm{-.01}$ &      0.31 &          $\pm.00$ &      0.31 &          $\pm.00$ \\
       heart &      0.15 &          $+.01$ &      0.17 &          $\bm{-.01}$ &      0.16 &          $\pm.00$ &      0.17 &          $\pm.00$ &      0.18 &          $\bm{-.01}$ &      0.26 &          $\bm{-.08}$ &      0.35 &          $\bm{-.15}$ \\
       housing &      0.17 &          $\bm{-.03}$ &      0.23 &          $\bm{-.05}$ &      0.22 &          $\bm{-.04}$ &      0.20 &          $\bm{-.02}$ &      0.22 &          $\bm{-.03}$ &      0.34 &          $\bm{-.12}$ &      0.41 &          $\bm{-.13}$ \\
       ionosphere &      0.14 &          $+05$ &      0.19 &          $\bm{-.05}$ &      0.20 &          $\bm{-.05}$ &      0.20 &          $\bm{-.03}$ &      0.21 &          $\bm{-.03}$ &      0.35 &          $\bm{-.13}$ &      0.54 &          $\bm{-.29}$ \\
       sonar &      0.27 &          $\pm.00$ &      0.29 &          $+.02$ &      0.29 &          $+.01$ &      0.34 &          $\bm{-.04}$ &      0.36 &          $\bm{-.03}$ &      0.43 &          $\bm{-.10}$ &      0.45 &          $\bm{-.05}$ \\
%              liver-dis. &      0.40 &          $\bm{-.02}$ &      0.40 &          $\bm{-.01}$ &      0.40 &          $\bm{-.02}$ &      0.43 &          $\bm{-.05}$ &      0.48 &          $\bm{-.09}$ &      0.54 &          $\bm{-.13}$ \\

\bottomrule
\end{tabular}
\caption{Test error for \textsc{sgd} and $\meanop$\textsc{sgd} over 25 trials of artificially corrupted datasets.}
\label{table:big}
\end{table*}

\iffalse
\subsection{Learning from label proportions}
\begin{algorithm}[t]
\caption{Laplacian Kernel Mean Map}\label{algoLKMM}
\begin{algorithmic}
\STATE \textbf{Input} ${\mathcal{S}}_j, \pi_j, j\in [n]$; $\lambda, \gamma > 0$; $\matrice{V}$;
\STATE 1: $K_\tildeXplusminus  \leftarrow (\Pi - \gamma \matrice{L})^{-1} K_{\matrice{B}}$\;
\STATE 2: $\tilde{\bm{\mu}} \leftarrow \sum_j {\hat{p}_j (\hat{{\pi}}_j \tilde{{\bm{b}}}^+_j - 
(1-\hat{{\pi}}_j) \tilde{{\bm{b}}}^-_j)}$ \;
\STATE 3: $\bm{\alpha}^\star \leftarrow \argmin_{\bm{\alpha}} \frac{1}{2} \expect_{\calS} \sum_{\sigma} [1-\sigma K\bm{\alpha}]_+ + \frac{1}{2} [1 - \meanop^{\top} \bm{\alpha}]_+ + \lambda \bm{\alpha}^T K \bm{\alpha}$\;
\STATE \textbf{Output} $h(\bm{x}) = \sum_{i \in [m]} \alpha^\star_i k(\bm{x}_i, \bm{x})$
\end{algorithmic}
\end{algorithm}

As a last application, we consider the setting of learning from label proportions. In \textsc{llp}, . Theorems \ref{th:factorization}, \ref{th:kernel} are the tools allowing us to lift the Laplacian Mean Map Algorithm \citep{pnrcAN} to non-parametric kernel models. A detailed experimental analysis of the algorithm is out of scope. Instead, we focus on showing  the advantages of using sparse non-parametric models over the original algorithm on some simulated datasets.

In order to show the versatility of our framework, we learn with non-differentiable losses and $L_2$ regularization with stochastic gradient descent.  $l(x) = \rho|x| -x + 1$
Convergence is guaranteed since $l$ is convex and $1$-Lipschitz; its differential set is 
\fi

\section{Discussion and conclusion}\label{sec:concl}
%There are many orthogonal directions to take for discussing our work. We walk through some of them briefly.
\textbf{Mean and covariance operators}~~The intuition behind the relevance of the mean operator becomes clear once we rewrite it as follows. 
\begin{lemma}\label{th:cov} Let $\pi_+ \defeq \expect_{\calS} 1\{y > 0\}$ be the positive label proportion of $\calS$. Then
$
\bm{\mu}_\mathcal{S} = \mathbb{C}\text{ov}_{\mathcal{S}}[\bm{x},y] + (2\pi_+ - 1) \expect_{\mathcal{S}}[\bm{x}]\:\:.\label{lemmacov}
$
\end{lemma}
(Proof in \ref{proof:cov}) We have come to the unsurprising fact that --when observations are centered-- the covariance $\mathbb{C}\text{ov}_{\mathcal{S}}[\bm{x},y]$ is what we need to know about the labels for learning linear models. The rest of the loss may be seen as a data-dependent regularizer. However, notice how the condition $\| \meanop_{\calD} \|_2 = 0$ does not implies $\mathbb{C}\text{ov}_{\mathcal{D}}[\bm{x},y] = 0$, which would make linear classification hard and limit Theorem \ref{th:noisy2}'s validity to degenerate cases. A kernelized version of this Lemma is given in \citep{shsfHS}. \\

\textbf{The generality of factorization}~~Factorization is ubiquitous for any (margin) loss, beyond the theory seen so far. A basic fact of real analysis supports it: a function $l$ is (uniquely) the sum of an even function $l_e$ and an odd $l_o$:
\begin{align*}
l(x) = \frac{1}{2} \left(l(x)+l(-x) + l(x)-l(-x) \right)= l_e(x) + l_o(x)\:\:.
\end{align*}
One can check that $l_e$ and $l_o$ are indeed even and odd (Figure \ref{fig:1}). This is actually all we need for the factorization of $l$.
\begin{theorem}[Factorization]\label{th:factorization2}
For any sample $\calS$ and hypothesis $h$ the empirical $l$-risk can be written as
\begin{align*}
R_{\calS, l}(h) = \frac{1}{2} \expect_{\calS} \left[ \sum\nolimits_{\sigma\in \mathcal{Y}} l(\sigma h(\bm{x})) \right] + \expect_{\calS} \Big[ l_o(y h(\bm{x})) \Big]
\end{align*}
 where $l_o(\cdot)$ is odd and $l_e(\cdot) \defeq \sum_{\sigma\in \mathcal{Y}} l(\sigma h(\cdot))$ is even and both uniquely defined.
\end{theorem}
 Its range of validity is exemplified by $01$ loss, a non-convex discontinuous piece-wise linear function, which factors as
\begin{align*}
l_e(x) &= 
\begin{array}{lc}
  \Big\{ 
    \begin{array}{cc}
	\frac{1}{2} & x \neq 0 \\
      	1 & \text{otherwise}
    \end{array},
\end{array}
l_o(x) &= -\frac{1}{2} \sgn(x) \:\:.
\end{align*}
It follows immediately that $\expect_{\calS} [l_o(\cdot)]$ is sufficient for $y$. However,
$l_o$ is a function of model $\bm{\theta}$. This defeats the purpose a of sufficient statistic, which we aim to be computable from data only and it is the main reason to define \lol s. The Factorization Theorem \ref{th:factorization} can also be stated for \textsc{rkhs}. To show that, notice that we satisfy all hypotheses of the Representer Theorem \citep{ssLW}.
\vspace{-5px}
\begin{theorem}\label{th:kernel}
Let $h(\bm{x}): \mathcal{X} \rightarrow \mathcal{H}$ be a feature map into a Reproducing Kernel Hilbert Space (\textsc{rkhs}) $\mathcal{H}$ with symmetric positive definite kernel $k: \mathcal{X} \times \mathcal{X} \rightarrow \mathbb{R}$, such that $h:\bm{x} \rightarrow k(\cdot, \bm{x})$. For any learning sample $\calS$, the empirical $l$-risk $R_{\calS, l}(h)$ with $\Omega: ||h||_{\mathcal{H}} \rightarrow \mathbb{R}^+$ regularization can be written as 
\begin{align*}
\frac{1}{2} \expect_{\calS} \left[ \sum\nolimits_{\sigma\in \mathcal{Y}} l(\sigma h(\bm{x})) \right] + \expect_{\calS} \Big[ l_o(y h(\bm{x})) \Big] + \Omega(||h||_{\mathcal{H}})
\end{align*}
and the optimal hypothesis admits a representation of the form $h(\bm{x}) = \sum_{i \in [m]} \alpha_i k(\bm{x}, \bm{x}_i)$.
\end{theorem}
All paper may be read in the context of non-parametric models, with the \emph{kernel} mean operator as sufficient statistic. Finally, it is simple to show factorization for square loss for regression (\ref{sec:regression}). This finding may open further applications of our framework. \\

\textbf{The linear-odd losses of \citep{pnsCF}}~~This recent work in the context of \textsc{pu} shows how the linear-odd condition on a \emph{convex} $l$ allows one to derive a tractable, \emph{i.e.} still convex, loss for learning with \textsc{pu}. The approach is conceptually related to ours as it isolates a label-free term in the loss, with the goal of leveraging on the unlabelled examples too. Interestingly, the linear term of their Equation 4 can be seen as a mean operator estimator like $\hat{\meanop} \defeq \mathbb{P}(y=1) \cdot \expect_{\calS_+}[\bm{x}]$, where $\calS_+$ is the set of positive examples. Their manipulation of the loss \emph{is not } equivalent to the Factorization Theorem \ref{th:factorization} though, as explained with details in (\ref{sec:pu}. Beside that, since we reason at the higher level of \textsc{wsl}, we can frame a solution for \textsc{pu} simply calling $\meanop$\textsc{sgd} on $\hat{\meanop}$ defined above or on estimators improved by exploiting results of \citep{pnrcAN}. \\

\textbf{Learning reductions}~~Solving a machine learning problem by solutions to other learning problems is a \emph{learning reduction} \citep{bdlmLR}. Our work \emph{does} fit into this framework. Following \citep{bdhlzEL}, we define a \textsc{wsl} task as a triple $(\mathcal{K}, \mathcal{Y}, l)$, with weakly supervised advice $\mathcal{K}$, predictions space $\mathcal{Y}$ and loss $l$, and we reduce to binary classification $(\mathcal{Y}, \mathcal{Y}, l)$. Our reduction is somehow simple, in the sense that $\mathcal{Y}$ does not change \emph{and neither does} $l$. Although, Algorithm \ref{algo:meta} modifies the internal code of the ``oracle learner" which contrasts with the concept of reduction. Anyway, we could as well write subgradients as 
\begin{align*}
\frac{1}{2} \left( \partial  l(\langle \bm{\theta}^{t}, \bm{x}_i \rangle) + \partial l (-\langle \bm{\theta}^{t}, \bm{x}_i \rangle) + a \meanop \right)  \:\:,
\end{align*}
which equals $\partial l$, and thus the oracle would be untouched. \\

\textbf{Beyond $\meanop$\textsc{sgd}}~~\textsc{meta}-$\meanop$\textsc{sgd} is intimately similar to stochastic average gradient (\textsc{sag}) \citep{srbMF}. Let denote $g_{e}^{i, t}(\bm{\theta}) \in \partial l_{e}(y_i\dotpi)$ if $i = i^{t}$ (example $i$ picked at time $t$), otherwise $= g_{e}^{i, t-1}(\bm{\theta})$. Define the same for $l_o$ accordingly. Then, \textsc{\textsc{sag}}'s model update is:
\begin{align*}
\bm{\theta}^{t+1} \leftarrow \bm{\theta}^t - \frac{\eta}{m} \sum\nolimits_{i \in [m]} g_{e}^{i, t}(\bm{\theta}^{t}) - \frac{\eta}{m} \sum\nolimits_{i \in [m]} g_{o}^{i, t}(\bm{\theta}^{t}) \:\:,
\end{align*}
 and recalling that $a\meanop_{\calS} = \expect_{\calS} [ \partial l_{o}(\bm{\theta})]$, $\meanop$\textsc{sgd}'s update is
\begin{align*}
\bm{\theta}^{t+1} \leftarrow \bm{\theta}^t - \eta~\partial l_{e}^{i}(\bm{\theta}^{t}) - \frac{\eta}{m} \sum\nolimits_{i \in [m]}  \partial l_{o}^{i}(\bm{\theta}^{t}) \:\:.
\end{align*}
From this parallel, the two algorithms appear to be variants of a more general sampling mechanism of examples \emph{and} gradient components, at each step. More generally, stochastic gradient is just \emph{one} suit of algorithms that fits into our 2-step learning framework. Proximal methods \citep{bjmoOW} are another noticeable example. The same \emph{modus operandi} leads to a proximal step of the form:
\begin{align*}
\bm{\theta}^{t+1} \leftarrow \prox_{\Theta} \left( \bm{\theta}^t + \eta \left( \partial R_{\calS_{2x}, l}(\bm{\theta}^t) +  \frac{a}{2} \meanop \right) \right)
\end{align*}
 with $\prox_{g}(x) = \argmin_{x'} g(x') + \frac{1}{2} \|x - x' \|_2^2)$ and $\Theta(\cdot)$ the regularizer. Once again, the adaptation works by summing $\meanop$ in the gradient step and changing the input to $\calS_{2x}$.  \\
 %Finally, we remark that both families of algorithms can actually be instantiated outside the realm of \lol s. As first order methods, the only requirement on $l$ is that for every ``piece" of its subdifferential set $\partial l$ is linear. For example, we can plug hinge loss into those algorithms, since its subgradients for any example $i$ is either $\bm{x}_iy_i$ or $0$.

\textbf{A better (?) picture of robustness}~~The data-dependent worst-case result of \citep{lsRC}, like any extreme-case argument, should be handled with care. It does not give the big picture for all data we may encounter in a real world, but only the most pessimistic. We present such a global view which appears better than expected: learning the minimizer from noisy data does not necessarily reduce convex losses to a singleton \citep{rmwLW} but depends on the mean operator for a large number of them (not necessarily linear, convex or smooth). Quite surprisingly, factorization also marries the two opposite views in one formula\footnote{See \citep[Theorem 1]{gmsMR}.}:
\begin{align*}
l(x) = \frac{1}{2} (~\underbrace{l(x)+l(-x)}_{ = \mathrm{const}~\Rightarrow~0\text{-\textsc{aln}}} + \underbrace{l(x)-l(-x)}_{= ax~\Rightarrow~\epsilon\text{-\textsc{aln}}}~)\:\:.
\end{align*} \\

To conclude, we have seen how losses factor in a way that we can isolate the \emph{contribution of supervision}. This has several implications both on theoretical and practical grounds: learning theory, formal analysis of label noise robustness, and adaptation of algorithms to handle poorly labelled data. An interesting question is whether factorization would let one identify what really matters in learning that is instead \emph{completely unsupervised}, and to do so with more complex models than the ones considered here, as for example deep architectures.

%\clearpage

\section*{Acknowledgements}
The authors thank Aditya Menon for insightful feedback on an earlier draft. NICTA is funded by the Australian Government through the Department of Communications and the Australian Research Council through the ICT Center of Excellence Program.

\bibliography{main}
\bibliographystyle{unsrtnat}

\appendix

\section*{Appendices}
\addcontentsline{toc}{section}{Appendices}

\section{Proofs}

\subsection{Proof of Lemma \ref{th:suff}}\label{proof:suff}
We need to show the double implication that defines sufficiency for $y$. \\
$\Rightarrow)$ By Factorization Theorem (\ref{th:factorization}), $R_{\calS,l}(h) - R_{\calS',l}(h)$ is label independent only if the odd part cancels out.\\
$\Leftarrow)$ If $\meanop_{\calS} = \meanop_{\calS}'$ then $R_{\calS,l}(h) - R_{\calS',l}(h)$ is independent of the label, because the label only appears in the mean operator due to Factorization Theorem (\ref{th:factorization}).

\subsection{Proof of Lemma \ref{th:1-to-1even}}\label{proof:1-to-1even}
Consider the class of \lol s satisfying $l(x) - l(-x) = 2ax$. For any element of the class, define $l_e(x)=l(x) - ax$, which is even. In fact we have $$l_e(-x)=l(-x)+ ax = l(x) - 2ax + ax = l(x) - ax = l_e(x)\:\:.$$

\subsection{Proof of Theorem \ref{th:generalization}}\label{proof:generalization}

We start by proving two helper Lemmas. The next one provides a bound to the Rademacher complexity computed on the sample $\calS_{2x} \defeq \{ (\bm{x}_i, \sigma), i \in [m], \forall \sigma \in \mathcal{Y} \}$.

\begin{lemma}\label{lemma1}
Suppose $m$ even. Suppose $\mathcal{X} = \{ \bm{x}: \| \bm{x} \|_2 \leq X \}$ be the observations space, and $\mathcal{H} = \{ \bm{\theta}: \|\bm{\theta}\|_2 \leq B \}$ be the space of linear hypotheses. Let $\mathcal{Y}^{2m} \defeq \times_{j \in [2m]} \mathcal{Y}$.
Then the empirical Rademacher complexity $$\mathcal{R}(\mathcal{H} \circ \mathcal{S}_{2x}) \defeq \expect_{\sigma \sim \mathcal{Y}^{2m}} \left[ \sup_{\bm{\theta} \in \mathcal{H}} \frac{1}{2m} \sum_{i \in [2m]} \sigma_i \langle \bm{\theta}, \bm{x}_i \rangle \right]$$ of $\mathcal{H}$ on $\calS_{2x}$ satisfies:
\begin{eqnarray}
\mathcal{R}(\mathcal{H} \circ \mathcal{S}_{2x}) & \leq & v \cdot \frac{BX}{\sqrt{2m}}\:\:,
\end{eqnarray}
with $ v \defeq \frac{1}{2} + \frac{1}{2}\sqrt{\frac{1}{2}-\frac{1}{m}}$.
\end{lemma}
\begin{proof}
Suppose without loss of generality that $\bm{x}_i = \bm{x}_{m+i}$. The proof relies on the observation that $\forall \bm{\sigma} \in
\mathcal{Y}^{2m}$,
\begin{eqnarray}
  \arg \sup_{\bm{\theta} \in {\mathcal{H}}}
 \left\{
   \expect_{{\mathcal{S}}}[\sigma(\bm{x}) \dotp]
   \right\} & = & \frac{1}{2m}\arg \sup_{\bm{\theta} \in {\mathcal{H}}}
 \left\{
   \sum_i {\sigma_i \dotpi}
   \right\}\nonumber \\
 & = & \frac{\sup_{{\mathcal{H}}}\|\bm{\theta}\|_2}{\|\sum_i {\sigma_i \bm{x}_i}\|_2} \sum_i {\sigma_i \bm{x}_i}\:\:.
\end{eqnarray}
So,
\begin{eqnarray}
  \mathcal{R}(\mathcal{H} \circ \mathcal{S}_{2x}) & = & \expect_{\mathcal{Y}^{2m}}
 \sup_{h \in {\mathcal{H}}}
 \left\{
   \expect_{{\mathcal{S}}_{2x}}[\sigma(\bm{x}) h(\bm{x})]
   \right\} \nonumber\\
 & = &\frac{\sup_{{\mathcal{H}}}\|\theta\|_2}{2m}\cdot \expect_{\mathcal{Y}^{2m}} \left[\frac{\left(\sum_{i=1}^{2m} {\sigma_i
     \bm{x}_i}\right)^\top \left(\sum_{i=1}^{2m} {\sigma_i
     \bm{x}_i}\right)}{\|\sum_{i=1}^{2m} {\sigma_i
     \bm{x}_i}\|_2}\right]\nonumber\\
 & = & \sup_{{\mathcal{H}}}\|\theta\|_2 \cdot \expect_{\mathcal{Y}^{2m}} \left[\frac{1}{2m} \cdot \left\|\sum_{i=1}^{2m} {\sigma_i
     \bm{x}_i}\right\|_2\right]\label{eq11}\:\:.
\end{eqnarray}
Now, remark that whenever $\sigma_i = -\sigma_{m+i}$, $\bm{x}_i$ disappears in the sum, and therefore the max norm for the sum may decrease as well. This suggests to split the $2^{2m}$ assignations into $2^m$ groups of size $2^m$, ranging over the possible number of observations taken into account in the sum. They can be factored by a weighted sum of contributions of each subset of indices ${\mathcal{I}} \subseteq [m]$ ranging over the non-duplicated observations:
\begin{eqnarray}
\expect_{\mathcal{Y}^{2m}} \left[\frac{1}{m} \cdot \left\|\sum_{i=1}^{2m} {\sigma_i
     \bm{x}_i}\right\|_2\right] & = & \frac{1}{2^{2m}} \sum_{{\mathcal{I}} \subseteq [m]} \frac{2^{m-|{\mathcal{I}}|}}{2m} \cdot \sum_{\bm{\sigma} \in \mathcal{Y}^{|{\mathcal{I}}|}} \sqrt{2} \left\|\sum_{i\in {\mathcal{I}}} {\sigma_i
     \bm{x}_i}\right\|_2\:\:.\label{eq111}\\
 & = & \frac{\sqrt{2}}{2^{m}} \sum_{{\mathcal{I}} \subseteq [m]} \frac{1}{2m}\cdot \underbrace{\frac{1}{2^{|{\mathcal{I}}|}} \cdot \sum_{\bm{\sigma} \in \mathcal{Y}^{|{\mathcal{I}}|}} \left\|\sum_{i\in {\mathcal{I}}} {\sigma_i
     \bm{x}_i}\right\|_2}_{u_{|{\mathcal{I}}|}}\:\:.\label{es001}
\end{eqnarray}
The $\sqrt{2}$ factor appears because of the fact that we now consider only the observations of ${\mathcal{S}}$. Now, for any \textit{fixed} ${\mathcal{I}}$, we renumber its observations in $[|{\mathcal{I}}|]$ for simplicity, and observe that, since $\sqrt{1+x} \leq 1 + x/2$,
\begin{eqnarray}
u_{|{\mathcal{I}}|} & = & \frac{1}{2^{|{\mathcal{I}}|}} \sum_{\bm{\sigma}\in \mathcal{Y}^{|{\mathcal{I}}|}} {\sqrt{\sum_{i\in {\mathcal{I}}} \|\bm{x}_i\|_2^2 + \sum_{i_1\neq i_2}
     {\sigma_{i_1} \sigma_{i_2} \bm{x}_{i_1}^\top \bm{x}_{i_2}}}}\\
 & = & \frac{\sqrt{\sum_{i\in {\mathcal{I}}} \|\bm{x}_i\|_2^2}}{2^{|{\mathcal{I}}|}} \sum_{\bm{\sigma}\in \mathcal{Y}^{|{\mathcal{I}}|}} \sqrt{1 + \frac{\sum_{i_1\neq i_2}
     {\sigma_{i_1} \sigma_{i_2} \bm{x}_{i_1}^\top \bm{x}_{i_2}}}{\sum_{i\in {\mathcal{I}}} \|\bm{x}_i\|_2^2}}\\
 & \leq & \frac{\sqrt{\sum_{i\in {\mathcal{I}}} \|\bm{x}_i\|_2^2}}{2^{|{\mathcal{I}}|}} \sum_{\bm{\sigma}\in \mathcal{Y}^{|{\mathcal{I}}|}} \left(1 + \frac{\sum_{i_1\neq i_2}
     {\sigma_{i_1} \sigma_{i_2} \bm{x}_{i_1}^\top \bm{x}_{i_2}}}{2\sum_{i\in {\mathcal{I}}} \|\bm{x}_i\|_2^2}\right) \\
 &  = & \sqrt{\sum_{i\in {\mathcal{I}}} \|\bm{x}_i\|_2^2} + \frac{1}{2^{|{\mathcal{I}}|} \cdot 2\sum_{i\in {\mathcal{I}}} \|\bm{x}_i\|_2^2} \cdot \sum_{\bm{\sigma}\in \mathcal{Y}^{|{\mathcal{I}}|}} \sum_{i_1\neq i_2}
     {\sigma_{i_1} \sigma_{i_2} \bm{x}_{i_1}^\top \bm{x}_{i_2}}\\
 &  = &\sqrt{\sum_{i\in {\mathcal{I}}} \|\bm{x}_i\|_2^2} + \frac{1}{2^{|{\mathcal{I}}|} \cdot 2\sum_{i\in {\mathcal{I}}} \|\bm{x}_i\|_2^2} \cdot \sum_{i_1\neq i_2}
     {\bm{x}_{i_1}^\top \bm{x}_{i_2}\cdot \underbrace{\left(\sum_{\bm{\sigma}\in \mathcal{Y}^{|{\mathcal{I}}|}} \sigma_{i_1} \sigma_{i_2}\right) }_{=0}}\\
 & = & \sqrt{\sum_{i\in {\mathcal{I}}} \|\bm{x}_i\|_2^2} \\
 & \leq & \sqrt{|\mathcal{I}|} \cdot X\:\:.
\end{eqnarray}
Plugging this in eq. (\ref{es001}) yields
\begin{eqnarray}
\frac{1}{X} \cdot \expect_{\mathcal{Y}^{2m}} \left[\frac{1}{m} \cdot \left\|\sum_{i=1}^{2m} {\sigma_i
     \bm{x}_i}\right\|_2\right] & \leq & \frac{\sqrt{2}}{2^{m}} \sum_{k = 0}^{m} \frac{\sqrt{k}}{2m} {m \choose k}\:\:.\label{es01}
\end{eqnarray}
Since $m$ is even:
\begin{eqnarray}
\expect_{\mathcal{Y}^{2m}} \left[\frac{1}{2m} \cdot \left\|\sum_{i=1}^{2m} {\sigma_i
     \bm{x}_i}\right\|_2\right] & \leq & \frac{\sqrt{2}}{2^{m}} \sum_{k = 0}^{(m/2)-1} \frac{\sqrt{k}}{2m} {m \choose k}+\frac{\sqrt{2}}{2^{m}} \sum_{k = m/2}^{m} \frac{\sqrt{k}}{2m} {m \choose k}\:\:.\label{eqq01}
\end{eqnarray}
Notice that the left one trivially satisfies
\begin{eqnarray}
\frac{\sqrt{2}}{2^{m}} \sum_{k = 0}^{(m/2)-1} \frac{\sqrt{k}}{2m} {m \choose
  k} & \leq & \frac{\sqrt{2}}{2^{m}} \sum_{k = 0}^{(m/2)-1} \frac{1}{2m}
\cdot \sqrt{\frac{m-2}{2}} {m \choose
  k} \nonumber\\
 & = & \frac{1}{2} \cdot \sqrt{\frac{1}{m} - \frac{2}{m^2}} \cdot \frac{1}{2^{m}} \sum_{k = 0}^{(m/2)-1} 
{m \choose
  k} \nonumber\\
 & \leq & \frac{1}{4} \cdot \sqrt{\frac{1}{m} - \frac{2}{m^2}}
\end{eqnarray}
Also, the right one satisfies:
\begin{eqnarray}
\frac{\sqrt{2}}{2^{m}} \sum_{k = m/2}^{m} \frac{\sqrt{k}}{2m} {m
  \choose k} & \leq & \frac{\sqrt{2}}{2^{m}} \sum_{k = m/2}^{m} \frac{\sqrt{m}}{2m} {m
  \choose k}\nonumber\\
 & = & \frac{1}{\sqrt{2m}} \cdot \frac{1}{2^{m}} \sum_{k = m/2}^{m} {m
  \choose k}\nonumber\\
 & = & \frac{1}{2} \cdot \frac{1}{\sqrt{2m}}\:\:.
\end{eqnarray}
We get
\begin{eqnarray}
\frac{1}{X} \cdot \expect_{\mathcal{Y}^{2m}} \left[\frac{1}{m} \cdot \left\|\sum_{i=1}^{2m} {\sigma_i
     \bm{x}_i}\right\|_2\right] & \leq & \frac{1}{4} \cdot \sqrt{\frac{1}{m} - \frac{2}{m^2}} + \frac{1}{2} \cdot \sqrt{\frac{1}{2m}} \\ 
 & = & \frac{1}{\sqrt{2m}} \cdot \left(\frac{1}{2} + \frac{1}{2}\sqrt{\frac{1}{2}-\frac{1}{m}}\right)\:\:.
\end{eqnarray}
And finally:
\begin{eqnarray}
  \mathcal{R}(\mathcal{H} \circ \mathcal{S}_{2x}) & \leq & v \cdot \frac{B X}{\sqrt{2m}}\:\:,
\end{eqnarray}
with 
\begin{eqnarray}
v & \defeq & \frac{1}{2} + \frac{1}{2}\sqrt{\frac{1}{2}-\frac{1}{m}}\:\:,
\end{eqnarray}
as claimed. 
\end{proof}

The second Lemma is a straightforward application of McDiarmid 's inequality \citep{mdC} to evaluate the convergence of the empirical mean operator to its population counterpart.

\begin{lemma}\label{lemma2}
Suppose $\mathbb{R}^d \supseteq \mathcal{X} = \{ \bm{x}: \| \bm{x} \|_2 \leq X < \infty \}$ be the observations space. Then for any $\delta > 0$ with probability at least $1 - \delta$
$$
\left\| \meanop_{\calD} - \meanop_{\calS} \right\|_2 \leq X \cdot \sqrt{\frac{d}{m} \log{ \left( \frac{d}{\delta} \right) }}\:\:.
$$
\end{lemma}
\begin{proof}
Let $\calS$ and $\calS'$ be two learning samples that differ for only one example $(\bm{x}_i, y_i) \neq (\bm{x}_{i'}, y_{i'})$. Let first consider the one-dimensional case. We refer to the $k$-dimensional component of $\meanop$ with $\meanop^k$. For any $\calS, \calS'$ and any $k \in [d]$ it holds
\begin{align*}
\left| \meanop^k_{\calS} - \meanop^k_{\calS'} \right|
&= \frac{1}{m} \left| \bm{x}^k_{i} y_{i} - \bm{x}^k_{i'} y_{i'} \right| \\
&\leq \frac{X}{m} \left|y_{i} - y_{i'} \right| \\
&\leq \frac{2X}{m}\:\:.
\end{align*}
This satisfies the bounded difference condition of McDiarmid's inequality, which let us write for any $k \in [d]$ and any $\epsilon > 0$ that
\begin{align*}
\mathbb{P} \left( \left| \meanop^k_{\calD} - \meanop^k_{\calS} \right| \geq \epsilon \right) \leq \exp \left(-\frac{m\epsilon^2}{2X^2} \right)
\end{align*}
and the multi-dimensional case, by union bound
\begin{align*}
\mathbb{P} \left(\exists k \in [d] : \left| \meanop^k_{\calD} - \meanop^k_{\calS} \right| \geq \epsilon \right) \leq d \exp \left(-\frac{m\epsilon^2}{2X^2} \right)\:\:.
\end{align*}
Then by negation
\begin{align*}
\mathbb{P} \left(\forall k \in [d] : \left| \meanop^k_{\calD} - \meanop^k_{\calS} \right| \leq \epsilon \right) \geq 1 - d\exp \left(-\frac{m\epsilon^2}{2X^2} \right)\:\:,
\end{align*}
which implies that for any $\delta > 0$ with probability $1 - \delta$
\begin{align*}
X \sqrt{\frac{2}{m} \log{\left(\frac{d}{\delta} \right)}} \geq \left\| \meanop_{\calD} - \meanop_{\calS} \right\|_{\infty} \geq d^{-1/2} \left\| \meanop_{\calD} - \meanop_{\calS} \right\|_2\:\:.
\end{align*}
This concludes the proof.
\end{proof}

We now restate and prove Theorem \ref{th:generalization}. \\

\textbf{Theorem 7}~~\textit{Assume $l$ is  $a$-\lol~and $L$-Lipschitz. Suppose $\R^d \supseteq \mathcal{X} = \{ \bm{x}: \| \bm{x} \|_2 \leq X < \infty \}$ be the observations space, and $\mathcal{H} = \{ \bm{\theta}: \|\bm{\theta}\|_2 \leq B < \infty \}$ be the space of linear hypotheses. Let $c(X, B) \defeq \max_{y \in \mathcal{Y}} l(y XB)$. Let $\hat{\bm{\theta}} = \argmin_{\bm{\theta} \in \mathcal{H}} R_{\calS, l}(\bm{\theta})$. Then for any $\delta > 0$, with probability at least $1 - \delta$
\begin{align*}
R_{\calD, l}(\hat{\bm{\theta}}) - R_{\calD, l}(\bm{\theta}^\star) \leq \left( \frac{\sqrt{2} + 1}{4} \right) \cdot \frac{XBL}{\sqrt{m}} + \frac{c(X, B) L}{2} \cdot \sqrt{\frac{1}{m}\log\left(\frac{1}{\delta}\right)} + 2|a|B \cdot \| \meanop_{\calD} - \meanop_{\calS}\|_2\:\:,
\end{align*}
or more explicitly
\begin{align*}
R_{\calD, l}(\hat{\bm{\theta}}) - R_{\calD, l}(\bm{\theta}^\star) \leq \left( \frac{\sqrt{2} + 1}{4} \right) \cdot \frac{XBL}{\sqrt{m}} + \left( \frac{c(X, B) L}{2} + 2|a|XB \sqrt{d \log{d}} \right) \sqrt{\frac{1}{m}\log\left(\frac{2}{\delta}\right)}\:\:.
\end{align*}} \\

\begin{proof}
Let $\bm{\theta}^\star = \argmin_{\bm{\theta} \in \mathcal{H}} R_{\calD, l}(\bm{\theta})$. We have
\begin{align}
R_{\calD, l}(\hat{\bm{\theta}}) - R_{\calD, l}(\bm{\theta}^\star) &= \frac{1}{2} R_{\calD_{2x}, l}(\hat{\bm{\theta}}) + a \langle \hat{\bm{\theta}}, \meanop_{\calD} \rangle
- \frac{1}{2} R_{\calD_{2x}, l}(\bm{\theta}^\star) - a \langle \bm{\theta}^\star, \meanop_{\calD} \rangle \label{generalization-step1} \\
&= \frac{1}{2} \left( R_{\calD_{2x}, l}(\hat{\bm{\theta}}) - R_{\calD_{2x}, l}(\bm{\theta}^\star) \right) + a \langle \hat{\bm{\theta}} - \bm{\theta}^\star, \meanop_{\calD} \rangle \nonumber\\
& = \frac{1}{2} \left( R_{\calS_{2x}, l}(\hat{\bm{\theta}}) - R_{\calS_{2x}, l}(\bm{\theta}^\star) \right) + a \langle \hat{\bm{\theta}} - \bm{\theta}^\star, \meanop_{\calD} \rangle \nonumber \\ 
&~+ \frac{1}{2} \Big(R_{\calD_{2x}, l}(\hat{\bm{\theta}}) - R_{\calS_{2x}, l}(\hat{\bm{\theta}}) - 
R_{\calD_{2x}, l}(\bm{\theta}^\star) + R_{\calS_{2x}, l}(\bm{\theta}^\star)\Big)~~\textbf{\Big\}} A_1\:\:. \label{generalization-step2}
\end{align}

Step \ref{generalization-step1} is obtained by the equality $R_{\calD, l}(\bm{\theta}) = \frac{1}{2} R_{\calD_{2x}, l}(\bm{\theta}) +  a \langle \bm{\theta}, \meanop_{\calD} \rangle$ for any $\bm{\theta}$. Now, rename Line \ref{generalization-step2} as $A_1$. Applying the same equality with regard to $\calS$, we have
\begin{align} 
R_{\calD, l}(\hat{\bm{\theta}}) - R_{\calD, l}(\bm{\theta}^\star) \leq \underbrace{R_{\calS, l}(\hat{\bm{\theta}}) - R_{\calS, l}(\bm{\theta}^\star)}_{A_2} + \underbrace{a \langle \hat{\bm{\theta}} - \bm{\theta}^\star, \meanop_{\calD} - \meanop_{\calS} \rangle}_{A_3} + A_1 \nonumber \:\:.
\end{align}
Now, $A_2$ is never more than $0$ because $\hat{\bm{\theta}}$ is the minimizer of $R_{\calS, l}(\bm{\theta})$. From the Cauchy-Schwarz inequality and bounded models it holds true that
\begin{align}
A_3 \leq |a| \left\| \hat{\bm{\theta}} - \bm{\theta}^\star \right\|_2 \cdot \Big\| \meanop_{\calD} - \meanop_{\calS} \Big\|_2 \leq 2|a|B \Big\| \meanop_{\calD} - \meanop_{\calS} \Big\|_2\:\:. \label{eq:ub-a2}
\end{align}

We could treat $A_1$ by calling standard bounds based on Rademacher complexity on a sample with size $2m$ \citep{bmRA}. Indeed, since the complexity does not depend on labels, its value would be the same --modulo the change of sample size-- for both $\calS$ and $\calS_{2x}$, as they are computed with same loss and observations. However, the special structure of $\calS_{2x}$ allows us to obtain a tighter structural complexity term, due to some cancellation effect. The fact is proven by Lemma \ref{lemma1}. In order to exploit it, we first observe that
\begin{align*}
A_1 &\leq \frac{1}{2} \Big(R_{\calD_{2x}, l}(\hat{\bm{\theta}}) - R_{\calS_{2x}, l}(\hat{\bm{\theta}}) - 
R_{\calD_{2x}, l}(\bm{\theta}^\star) + R_{\calS_{2x}, l}(\bm{\theta}^\star)\Big) \\
&\leq \sup_{\bm{\theta} \in \mathcal{H}} \left| R_{\calD_{2x}, l}(\bm{\theta}) - R_{\calS_{2x}, l}(\bm{\theta}) \right|
\end{align*}
which by standard arguments \citep{bmRA} and the application of Lemma \ref{lemma1} gives a bound with probability at least $1 - \delta$, $\delta > 0$
\begin{align*}
A_1 & \leq 2 L \cdot \mathcal{R}(\mathcal{H} \circ \calS_{2x}) + c(X, B) L \cdot \sqrt{\frac{1}{4m}\log\left(\frac{1}{\delta}\right)} \\
& \leq L \cdot \frac{\sqrt{2} + 1}{\sqrt{2}} \cdot \frac{BX}{\sqrt{2m}}+ c(X, B) L \cdot \sqrt{\frac{1}{4m}\log\left(\frac{1}{\delta}\right)}
\end{align*}
where $c(X, B) \defeq \max_{y \in \mathcal{Y}} l(yXB)$ and because $\frac{1}{2} + \frac{1}{2}\sqrt{\frac{1}{2} - \frac{1}{m}} < \left( \frac{\sqrt{2} + 1}{\sqrt{2}} \right),~ \forall m > 0$. We combine the results and get with probability at least $1 - \delta$, $\delta > 0$ that
\begin{align}
R_{\calD, l}(\hat{\bm{\theta}}) - R_{\calD, l}(\bm{\theta}^\star) \leq \left( \frac{\sqrt{2} + 1}{2} \right) \cdot \frac{XBL}{\sqrt{m}} + \frac{c(X, B) L}{2} \cdot \sqrt{\frac{1}{m}\log\left(\frac{1}{\delta}\right)} + 2|a|B \cdot \| \meanop_{\calD} - \meanop_{\calS}\|_2\:\:. \label{eq:first-statement}
\end{align}
This proves the first part of the statement. For the second one, we apply Lemma \ref{lemma2} that provides the probabilistic bound for the norm discrepancy of the mean operators. Consider that both statements are true with probability at least $1 - \delta/2$. We write
\begin{align*}
\mathbb{P} \Bigg(& \left\{ R_{\calD, l}(\hat{\bm{\theta}}) - R_{\calD, l}(\bm{\theta}^\star ) \leq \left( \frac{\sqrt{2} + 1}{2} \right) \cdot \frac{XBL}{\sqrt{m}} + \frac{c(X, B) L}{2} \cdot \sqrt{\frac{1}{m}\log\left(\frac{2}{\delta}\right)} + 2|a|B \cdot \| \meanop_{\calD} - \meanop_{\calS}\|_2 \right\} \\
& \bigwedge \left\{ \left\| \meanop_{\calD} - \meanop_{\calS} \right\|_2 \leq X \cdot \sqrt{\frac{d}{m} \log{ \left( \frac{2d}{\delta} \right) }} \right\} \Bigg) \geq 1 - \delta/2 - \delta/2 = 1 - \delta\:\:,
\end{align*}
and therefore with probability $1 - \delta$
\begin{align*}
R_{\calD, l}(\hat{\bm{\theta}}) - R_{\calD, l}(\bm{\theta}^\star ) 
&\leq \left( \frac{\sqrt{2} + 1}{2} \right) \cdot \frac{XBL}{\sqrt{m}} + \frac{c(X, B) L}{2} \cdot \sqrt{\frac{1}{m}\log\left(\frac{2}{\delta}\right)} + 
2|a|XB \cdot \sqrt{\frac{d}{m} \log{ \left( \frac{2d}{\delta} \right) }} \\
& = \left( \frac{\sqrt{2} + 1}{2} \right) \cdot \frac{XBL}{\sqrt{m}} + \left( \frac{c(X, B) L}{2} + 2|a|XB \sqrt{d \log{d}} \right) \sqrt{\frac{1}{m}\log\left(\frac{2}{\delta}\right)}\:\:.
\end{align*}
\end{proof}

\subsection{Unbiased estimator for the mean operator with asymmetric label noise}\label{proof:unbiased}

\citep[Lemma 1]{ndrtLW} provides an unbiased estimator for a loss $l(x)$ computed on $x$ of the form:
\begin{align*}
\hat{l}(y\dotpi) &\defeq \frac{(1 - p_{-y}) \cdot l(\dotpi) + p_y \cdot l(-\dotpi)}{1 - p_- - p_+}
\end{align*}
We apply it for estimating the mean operator instead of, from another perspective, for estimating a linear (unhinged) loss as in \citep{rmwLW}. We are allowed to do so by the very result of the Factorization Theorem, since the noise corruption has effect on the linear-odd term of the loss only. The estimator of the sufficient statistic of a single example $y\bm{x}$ is
\begin{align*}
\hat{\bm{z}} &\defeq \frac{1 - p_{-y} + p_y}{1 - p_- - p_+} y\bm{x} \\
&= \frac{1 - (p_- - p_+)y}{1 - p_- - p_+} y\bm{x} \\
&= \frac{y - (p_- - p_+)}{1 - p_- - p_+} \bm{x}\:\:,
\end{align*}
and its average, \emph{i.e.} the mean operator estimator, is
\begin{align*}
\hat{\meanop}_{\calS} \defeq \expect_{\calS} \left[ \frac{y - (p_- + p_+)}{1 - p_- - p_+} \bm{x} \right]\:\:,
\end{align*}
such that in expectation over the noisy distribution it holds $\expect_{\tilde{\calD}}[\hat{\bm{z}}] = \bm{\mu}_{\calD}.$ Moreover, the corresponding risk enjoys the same unbiasedness property. In fact
\begin{align}
\hat{R}_{\tilde{\calD}, l}(\bm{\theta}) &= \frac{1}{2} R_{\calD_{2x}, l}(\bm{\theta}) + \expect_{\tilde{\calD}} \left[ a\langle \bm{\theta}, \hat{\bm{z}} \rangle \right] \nonumber \\
& = \frac{1}{2} R_{\calD_{2x}, l}(\bm{\theta}) + a\langle \bm{\theta}, \hat{\meanop}_{\tilde{\calD}} \rangle \\
& = \frac{1}{2} R_{\calD_{2x}, l}(\bm{\theta}) + a\langle \bm{\theta}, \meanop_{\calD} \rangle  \nonumber\\
&= R_{\calD, l}(\bm{\theta}) \nonumber \:\:,
\end{align}
where we have also used the independency on labels (and therefore of label noise) of $R_{\calD_{2x}, l}$.

\subsection{Proof of Theorem \ref{th:noisy1}}\label{proof:generalization-noisy}

This Theorem is a version of Theorem \ref{th:generalization} applied to the case of asymmetric label noise. Those results differ in three elements. First, we consider the generalization property of a minimizer $\hat{\bm{\theta}}$ that is learnt on the corrupted sample $\tilde{\calS}$. Second, the minimizer is computed on the basis of the unbiased estimator of $\hat{\meanop}_{\tilde{\calS}}$ and not barely $\meanop_{\tilde{\calS}}$. Third, as a consequence, Lemma \ref{lemma2} is not valid in this scenario. Therefore, we first prove a version of the bound for the mean operator norm discrepancy while considering label noise.

\begin{lemma}\label{lemma2-noisy}
Suppose $\mathbb{R}^d \supseteq \mathcal{X} = \{ \bm{x}: \| \bm{x} \|_2 \leq X < \infty \}$ be the observations space.
Let $\tilde{\calS}$ is a learning sample affected by asymmetric label noise with noise rates $(p_+, p_-) \in [0, 1/2)$.
Then for any $\delta > 0$ with probability at least $1 - \delta$
$$
\left\| \hat{\meanop}_{\tilde{\calD}} - \hat{\meanop}_{\tilde{\calS}} \right\|_2 \leq \frac{X}{1 - p_- - p_+} \cdot \sqrt{\frac{d}{m} \log{ \left( \frac{d}{\delta} \right) }}\:\:.
$$
\end{lemma}
\begin{proof}
Let $\tilde{\calS}$ and $\tilde{\calS'}$ be two learning samples from the corrupted distribution $\tilde{\calD}$ that differ for only one example $(\bm{x}_i, \tilde{y}_i) \neq (\bm{x}_{i'}, \tilde{y}_{i'})$. Let first consider the one-dimensional case. We refer to the $k$-dimensional component of $\meanop$ with $\meanop^k$. For any $\tilde{\calS}, \tilde{\calS}'$ and any $k \in [d]$ it holds
\begin{align*}
\left| \hat{\meanop}^k_{\tilde{\calS}} - \hat{\meanop}^k_{\tilde{\calS}'} \right|
&= \frac{1}{m} \left| \left( \frac{\tilde{y}_{i}  - (p_- - p_+)}{1 - p_- - p_+} \right) \bm{x}^k_{i}
- \left( \frac{\tilde{y}_{i'}  - (p_- - p_+)}{1 - p_- - p_+}  \right)  \bm{x}^k_{i'} \right| \\
&= \frac{1}{m} \left|  \frac{\tilde{y}_{i} \bm{x}^k_{i}}{1 - p_- - p_+}
-  \frac{\tilde{y}_{i'} \bm{x}^k_{i'} }{1 - p_- - p_+} \right| \\
&\leq \frac{X}{m(1 - p_- - p_+)} \left| \tilde{y}_{i} - \tilde{y}_{i'} \right| \\
& \leq \frac{2X}{m(1 - p_- - p_+)}\:\:.
\end{align*}
This satisfies the bounded difference condition of McDiarmid's inequality, which let us write for any $k \in [d]$ and any $\epsilon > 0$ that
\begin{align*}
\mathbb{P} \left( \left| \hat{\meanop}^k_{\calD} - \hat{\meanop}^k_{\calS} \right| \geq \epsilon \right) \leq \exp \left(-(1 - p_- - p_+)^2\frac{m\epsilon^2}{2X^2} \right)
\end{align*}
and the multi-dimensional case, by union bound
\begin{align*}
\mathbb{P} \left(\exists k \in [d] : \left| \hat{\meanop}^k_{\calD} - \hat{\meanop}^k_{\calS} \right| \geq \epsilon \right) \leq d \exp \left(-(1 - p_- - p_+)^2\frac{m\epsilon^2}{2X^2} \right)\:\:.
\end{align*}
Then by negation
\begin{align*}
\mathbb{P} \left(\forall k \in [d] : \left| \hat{\meanop}^k_{\calD} - \hat{\meanop}^k_{\calS} \right| \leq \epsilon \right) \geq 1 - d\exp \left(-(1 - p_- - p_+)^2\frac{m\epsilon^2}{2X^2} \right)\:\:,
\end{align*}
which implies that for any $\delta > 0$ with probability $1 - \delta$
\begin{align*}
\frac{X}{(1 - p_- - p_+)} \sqrt{\frac{2}{m} \log{\left(\frac{d}{\delta} \right)}} \geq \left\| \hat{\meanop}_{\calD} - \hat{\meanop}_{\calS} \right\|_{\infty} \geq d^{-1/2} \left\| \meanop_{\calD} - \meanop_{\calS} \right\|_2\:\:.
\end{align*}
This concludes the proof.
\end{proof}

The proof of Theorem \ref{th:noisy1} follows the structure of Theorem \ref{th:generalization}'s and elements of \citep[Theorem 3]{ndrtLW}'s. Let $\hat{\bm{\theta}} = \argmin_{\bm{\theta} \in \mathcal{H}} \hat{R}_{\tilde{\calD}, l}(\bm{\theta})$ and $\bm{\theta}^\star = \argmin_{\bm{\theta} \in \mathcal{H}} R_{\calD, l}(\bm{\theta})$. We have
\begin{align}
R_{\calD, l}(\hat{\bm{\theta}}) - R_{\calD, l}(\bm{\theta}^\star) & = \hat{R}_{\tilde{\calD}, l}(\hat{\bm{\theta}}) - \hat{R}_{\tilde{\calD}, l}(\bm{\theta}^\star)\label{proof:generalization-noisy1} \\
&= \frac{1}{2} R_{\calD_{2x}, l}(\hat{\bm{\theta}}) + a \langle \hat{\bm{\theta}}, \hat{\meanop}_{\tilde{\calD}} \rangle
- \frac{1}{2} R_{\calD_{2x}, l}(\bm{\theta}^\star) - a \langle \bm{\theta}^\star,  \hat{\meanop}_{\tilde{\calD}} \rangle  \nonumber  \\
&= \frac{1}{2} \left( R_{\calD_{2x}, l}(\hat{\bm{\theta}}) - R_{\calD_{2x}, l}(\bm{\theta}^\star) \right) + a \langle \hat{\bm{\theta}} - \bm{\theta}^\star,  \hat{\meanop}_{\tilde{\calD}} \rangle \nonumber\\
& = \frac{1}{2} \left( R_{\calS_{2x}, l}(\hat{\bm{\theta}}) - R_{\calS_{2x}, l}(\bm{\theta}^\star) \right) + a \langle \hat{\bm{\theta}} - \bm{\theta}^\star, \hat{\meanop}_{\tilde{\calD}} \rangle \nonumber \\ 
&~+ \frac{1}{2} \Big(R_{\calD_{2x}, l}(\hat{\bm{\theta}}) - R_{\calS_{2x}, l}(\hat{\bm{\theta}}) - 
R_{\calD_{2x}, l}(\bm{\theta}^\star) + R_{\calS_{2x}, l}(\bm{\theta}^\star)\Big)~~\textbf{\Big\}} A_1\:\:. \label{generalization-noisy2}
\end{align}

Step \ref{proof:generalization-noisy1} is due to unbiasedness shown in Section \ref{proof:unbiased}. Again, rename Line \ref{generalization-noisy2} as $A_1$, which this time is bounded directly by Theorem \ref{th:generalization}. Next, we proceed as within the proof of Theorem \ref{th:generalization} but now exploiting the fact that $\frac{1}{2} R_{\calS_{2x}, l}(\bm{\theta}) = \hat{R}_{\tilde{\calS}, l}(\bm{\theta}) - a \langle \bm{\theta}, \hat{\meanop}_{\tilde{\calD}} \rangle$
\begin{align} 
R_{\calD, l}(\hat{\bm{\theta}}) - R_{\calD, l}(\bm{\theta}^\star) \leq \underbrace{\hat{R}_{\tilde{\calS}, l}(\hat{\bm{\theta}}) - \hat{R}_{\tilde{\calS}, l}(\bm{\theta}^\star)}_{A_2} + \underbrace{a \langle \hat{\bm{\theta}} - \bm{\theta}^\star, \hat{\meanop}_{\tilde{\calD}} - \hat{\meanop}_{\tilde{\calS}} \rangle}_{A_3} + A_1 \nonumber \:\:.
\end{align}
Now, $A_2$ is never more than $0$ because $\hat{\bm{\theta}}$ is the minimizer of $\hat{R}_{\tilde{\calS}, l}(\bm{\theta})$. From the Cauchy-Schwarz inequality and bounded models it holds true that
\begin{align}
A_3 \leq |a| \left\| \hat{\bm{\theta}} - \bm{\theta}^\star \right\|_2 \cdot \Big\| \hat{\meanop}_{\tilde{\calD}} - \hat{\meanop}_{\calS} \Big\|_2 \leq 2|a|B \Big\| \hat{\meanop}_{\tilde{\calD}} - \hat{\meanop}_{\tilde{\calS}} \Big\|_2 \label{eq:ub-a2}\:\:,
\end{align}
for which we can call Lemma \ref{lemma2-noisy}. Finally, by a union bound we get that for any $\delta > 0$ with probability $1 - \delta$
\begin{align*}
R_{\calD, l}(\hat{\bm{\theta}}) - R_{\calD, l}(\bm{\theta}^\star ) 
\leq \left( \frac{\sqrt{2} + 1}{2} \right) \cdot \frac{XBL}{\sqrt{m}} + \left( \frac{c(X, B) L}{2} + \frac{2|a|XB}{1 - p_+ -p_-} \sqrt{d \log{d}} \right) \sqrt{\frac{1}{m}\log\left(\frac{2}{\delta}\right)}\:\:.
\end{align*}

\subsection{Proof of Theorem \ref{th:noisy2}}\label{proof:noisy1}
We now restate and prove Theorem \ref{th:noisy1}. The reader might question the bound for the fact that the quantity on the right-hand side can change by rescaling $\meanop_{\calD}$ by $X$, \emph{i.e.} the max $L_2$ norm of observations in the space $\mathcal{X}$. Although, such transformation would affect $l$-risks on the left-hand side as well, balancing the effect. With this in mind, we formulate the result without making explicit dependency on $X$. \\

\textbf{Theorem 10}~~\textit{Assume $\{ \bm{\theta} \in \mathcal{H}: ||\bm{\theta}||_2 \leq B \}$. Let $(\bm{\theta}^\star, \tilde{\bm{\theta}}^\star)$ respectively the minimizers of $(R_{\calD, l}(\bm{\theta}), R_{\tilde{\calD}, l}(\bm{\theta}))$ in $\mathcal{H}$. Then every a-\lol~is $\epsilon$-\textsc{aln}. That is
\begin{align*}
R_{\tilde{\calD}, l}(\bm{\theta}^\star) - R_{\tilde{\calD}, l}(\tilde{\bm{\theta}}^\star)
\leq 4|a| B \max (p_-, p_+) \cdot \|\meanop_{\calD}\|_2\:\:.
\end{align*}
Moreover: \\
1. If $\| \meanop_{\calD} \|_2 = 0$ for $\calD$ then every \lol~is \textsc{aln} for any $\tilde{\calD}$. \\
2. Suppose that $l$ is also once differentiable and $\gamma$-strongly convex. Then $\| \bm{\theta}^\star - \tilde{\bm{\theta}}^\star \|^2_2 \leq 2 \epsilon / \gamma$ .}\\

\begin{proof}
The proof draws ideas from \citep{msNT}. Let us first assume the noise to be symmetric, \emph{i.e.} $p_+ = p_- = p$. For any $\bm{\theta}$ we have
\begin{align} 
R_{\tilde{\calD}, l}(\bm{\theta}^\star) - R_{\tilde{\calD}, l}(\bm{\theta}) = &~(1 - p) \left( R_{\calD, l}(\bm{\theta}^\star) - R_{\calD, l}(\bm{\theta})\right) \nonumber \\ 
&+ p \left(R_{\calD, l}(\bm{\theta}^\star) - R_{\calD, l}(\bm{\theta}) + 2a \langle \bm{\theta}^\star - \bm{\theta}, \meanop_{\calD} \rangle \right) \label{eq:00} \\
\leq &~\left(R_{\calD, l}(\bm{\theta}^\star) - R_{\calD, l}(\bm{\theta}) \right) + 4|a|Bp \| \meanop_{\calD} \|_2 \label{eq:ineq} \\
\leq &~4|a|Bp \| \meanop_{\calD} \|_2\:\:. \label{eq:0}
\end{align}
We are working with \lol s, which are such that $l(x) = l(-x) + 2ax$ and therefore we can take Step \ref{eq:00}. Step \ref{eq:ineq} follows from Cauchy-Schwartz inequality and bounded models. Step \ref{eq:0} is true because $\bm{\theta}^\star$ is the minimizer of $R_{\calD, l}(\bm{\theta})$. We have obtained a bound for any $\bm{\theta}$ and so for the supremum with regard to $\bm{\theta}$. Therefore:
$$
\sup_{\bm{\theta} \in \mathcal{H}} \left( R_{\tilde{\calD}, l}(\bm{\theta}^\star) - R_{\tilde{\calD}, l}(\bm{\theta}) \right) = R_{\tilde{\calD}, l}(\bm{\theta}^\star) - R_{\tilde{\calD}, l}(\tilde{\bm{\theta}}) \:\:.
$$
To lift the discussion to \emph{asymmetric} label noise, risks have to be split into losses for negative and positive examples. Let $R_{\calD^+, l}$ be the risk computed over the distribution of the positive examples $\calD^+$ and $R_{\calD^{-}, l}$ the one of the negatives, and denote the mean operators $\meanop_{\calD^+}, \meanop_{\calD^-} $ accordingly. Also, define the probability of positive and negative labels in $\calD$ as $\pi_\pm = \mathbb{P}(y = \pm1)$. The same manipulations for the symmetric case let us write
\begin{align*}
R_{\tilde{\calD}, l}(\bm{\theta}^\star) - R_{\tilde{\calD}, l}(\bm{\theta}) = &~ \pi_- \left( R_{\calD^-, l}(\bm{\theta}^\star) - R_{\calD^-, l}(\bm{\theta})\right) 
+ \pi_+ \left( R_{\calD^+, l}(\bm{\theta}^\star) - R_{\calD^+, l}(\bm{\theta})\right) \nonumber \\ 
& + 2a p_- \pi_- \langle \bm{\theta}^\star - \bm{\theta}, \meanop_{\calD^-} \rangle + 2a p_+ \pi_+ \langle \bm{\theta}^\star - \bm{\theta}, \meanop_{\calD^+} \rangle \\
\leq &~ \left( R_{\calD, l}(\bm{\theta}^\star) - R_{\calD, l}(\bm{\theta})\right)
+ 2a \langle \bm{\theta}^\star - \bm{\theta}, p_- \meanop_{\calD^-} + p_+ \meanop_{\calD^+} \rangle \nonumber \\
\leq &~4|a| B \cdot \| p_-\pi_- \meanop_{\calD^-} + p_+\pi_+ \meanop_{\calD^+} \|_2 \nonumber \\
\leq &~4|a| B \max(p_-, p_+) \cdot \| \pi_- \meanop_{\calD^-} + \pi_+ \meanop_{\calD^+} \|_2 \nonumber \\
= &~4|a| B \max(p_-, p_+) \cdot \| \meanop_{\calD}\|_2\:\:. \nonumber 
\end{align*}
Then, we conclude the proof by the same argument for the symmetric case. The first corollary is immediate. For the second, we first recall the definition of a function $f$ strongly convex.
\begin{definition}
A differentiable function $f(x)$ is $\gamma$-strongly convex if for all $x, x' \in Dom(f)$ we have
$$
f(x) - f(x') \geq \langle \nabla f(x'), x - x' \rangle + \frac{\gamma}{2} \left\| x - x' \right\|^2_2\:\:.
$$
\end{definition}
If $l$ is differentiable once and $\gamma$-strongly convex in the $\bm{\theta}$ argument, so it the risk $R_{\tilde{D}, l}$ by composition with linear functions. Notice also that $\nabla R_{\tilde{\calD}, l}(\tilde{\bm{\theta}}^\star) = 0$ because $\tilde{\bm{\theta}}^\star$ is the minimizer. Therefore:
\begin{align*}
\epsilon &\geq R_{\tilde{\calD}, l}(\bm{\theta}^\star) - R_{\tilde{\calD}, l}(\tilde{\bm{\theta}}^\star) \\
& \geq \left\langle \nabla R_{\tilde{\calD}, l}(\tilde{\bm{\theta}}^\star), \bm{\theta}^\star - \tilde{\bm{\theta}}^\star \right\rangle + \frac{\gamma}{2} \left\|  \bm{\theta}^\star - \tilde{\bm{\theta}}^\star  \right\|^2_2 \\
& \geq \frac{\gamma}{2} \left\|  \bm{\theta}^\star - \tilde{\bm{\theta}}^\star  \right\|^2_2\:\:, \\
\end{align*}
which means that
$$
\left\|  \bm{\theta}^\star - \tilde{\bm{\theta}}^\star  \right\|^2_2 \leq \frac{2 \epsilon}{\gamma}\:\:.
$$
\end{proof}

\subsection{Proof of Lemma \ref{th:cov}}\label{proof:cov}
\begin{align*}
\mathbb{C}\text{ov}_{\mathcal{S}}[\bm{x},y] &= \expect_{\mathcal{S}}[y\bm{x}] - \expect_{\mathcal{S}}[y]\expect_{\mathcal{S}}[\bm{x}] \nonumber \\
&= \bm{\mu}_\mathcal{S} - \left( \frac{1}{m} \sum_{i: y_i > 0}  1 - \frac{1}{m} \sum_{i: y_i < 0} 1 \right) \expect_{\mathcal{S}}[\bm{x}] \nonumber \\
&= \bm{\mu}_\mathcal{S} - \left( 2\pi_+ - 1 \right) \expect_{\mathcal{S}}[\bm{x}]\:\:. \nonumber
\end{align*}
The second statement follows immediately.

\section{Factorization of non linear-odd losses}\label{proof:upper}

When $l_o$ is not linear, we can find upperbounds in the form of affine functions. It suffices to be continuous and have asymptotes at $\pm \infty$.

\begin{lemma}\label{th:asympt1}
Let the loss $l$ be continuous. Suppose that it has asymptotes at $\pm \infty$, \emph{i.e.} there exist $c_1, c_2 \in \R$ and $d_1, d_2 \in \R$ such that
\begin{align*}
\lim_{x\rightarrow +\infty} f(x) - c_1 x - d_1 = 0, \:\:
\lim_{x\rightarrow -\infty} f(x) - c_2 x - d_2 = 0
\end{align*}
then there exists $q \in \R$ such that
$
l_o(x) \leq \frac{c_1 + c_2}{2}x + q \:\:.
$
\end{lemma}
\begin{proof}
One can compute the limits at infinity of $l_o$ to get
\begin{displaymath}
\lim_{x\rightarrow +\infty} l_o(x) - \frac{c_1 + c_2}{2}x =  \frac{d_1-d_2}{2} 
\end{displaymath}
and 
\begin{displaymath}
\lim_{x\rightarrow -\infty} l_o(x) - \frac{c_1 + c_2}{2}x = \frac{d_2-d_1}{2} \: \: .
\end{displaymath}
Then $q \defeq \sup\{l_o(x) -  \frac{c_1 + c_2}{2}x \} <+\infty$ as $l_o$ is continuous. Thus $l_o(x) - \frac{c_1 + c_2}{2}x \leq   q$.
\end{proof}

The Lemma covers many cases of practical interest outside the class of \lol s, \emph{e.g.} hinge, absolute and Huber losses. Exponential loss is the exception since $l_o(x) = -\mathrm{sinh}(x)$ cannot be bounded. Consider now hinge loss: $l(x) = [1 - x]_+$ is not differentiable in 1 nor proper \citep{rwCB}, however it is continuous with asymptotes at $\pm \infty$. Therefore, for any $\bm{\theta}$ its empirical risk is bounded as
\begin{align*}\label{eq:ub-hinge}
R_{\calS, hinge}(\bm{\theta}) \leq \frac{1}{2} R_{\calS_{2x}, hinge}(\bm{\theta}) - \frac{1}{2} \dotmu + q\:\:,
\end{align*}
since $c_1 =0$ and $c_2 = 1$. An alternative proof of this result on hinge is provided next, giving the exact value of $q=1/2$. The odd term for hinge loss is
\begin{align*}
l_o(x) &= \frac{1}{2} \left( [1-x]_+ - [1+x]_+ \right) \nonumber \\
&= \frac{1}{4} \left( -2x + |1-x| - |1+x| \right)
\end{align*}
 due to an arithmetic trick for the $\max$ function: $\max(a,b) = (a+b)/2 + |b-a|/2$. Then for any $x$
\begin{align*}
|1-x| &\leq |x| + 1, \\
|1+x| &\geq |x| - 1
\end{align*}
 and therefore
$$
l_o(x) \leq \frac{1}{4} (-2x + |x| + 1 - |x| +1) = \frac{1}{2}(1-x) \:\:.
$$

We also provide a ``if-and-only-if" version of Lemma \ref{th:asympt1} fully characterizing which family of losses can be upperbounded by a \lol.

\begin{lemma}
Let $l: \R \rightarrow \R$ a continuous function. Then there exists $c_1, d_1, d_2 \in \R$ such that
\begin{equation}\label{asy1}
\limsup_{x\rightarrow +\infty} l_o(x) - c_1 x - d_1= 0
\end{equation}
and
\begin{equation}\label{asy2}
\limsup_{x\rightarrow -\infty} l_o(x) - c_1 x - d_2= 0\:\:,
\end{equation}
if and only if there exists $q, q' \in \R$ such that $
l_o(x) \leq q'x + q$ for every $x\in \R$.
\end{lemma}

\begin{proof}
$\Rightarrow)$ Suppose that such limits exist and they are zero for some $c_1, d_1, d_2$. Let prove that $l_o$ is bounded from above by a line.\\
\begin{displaymath}
q = \sup_{x\in \R} \left\{l_o(x) - c_1 x\right\}  <\infty\:\:,
\end{displaymath}
because $l_o$ is continuous. So for every $x\in \R$
\begin{displaymath}
l_o(x) \leq c_1 x  + q\:\:.
\end{displaymath}
In particular we can take $c_1$ as the angular coefficient of the line.\\
$\Leftarrow)$ Vice versa we proceed by contradiction. Suppose that there exists $q,q' \in \R$ such that $l_o$ is bounded from above by $l(x) = q'x + q$. Suppose in addition that the conditions on the asymptotes $(\ref{asy1})$ and $(\ref{asy2})$ are false. This implies either the existence of a sequence $x_n \rightarrow +\infty$ such that
\begin{displaymath}
\lim_{n\rightarrow \infty} l_o(x_n) - q' x_n  \rightarrow \pm\infty\:\:,
\end{displaymath}  
or the existence of another sequence $x'_n \rightarrow -\infty$
\begin{displaymath}
\lim_{n\rightarrow \infty} l_o(y_n) - q' x'_n  \rightarrow \pm\infty \:\:.
\end{displaymath}  
On one hand, if at least one of these two limits is $+\infty$ then we already reach a contradiction, because $l_o(x)$ is supposed to be bounded from above by $l(x) = q'x + q$. Suppose on the other hand that $x_n \rightarrow +\infty$ is such that
\begin{displaymath}
\lim_{n\rightarrow +\infty} l_o(x_n) - q' x_n  \rightarrow -\infty\:\:.
\end{displaymath}  
Then defining $x'_n = -x_n$ we have 
\begin{displaymath}
\lim_{n\rightarrow +\infty} l_o(w_n) - m x'_n  \rightarrow +\infty\:\:,
\end{displaymath}  
and for the same reason as above we reach a contradiction.
\end{proof}

\section{Factorization of square loss for regression}\label{sec:regression}

We have formulated the Factorization Theorem for classification problems. However, a similar property holds for regression with square loss: $f(\dotpi, y) = (\dotpi - y_i)^2 $ factors as
\begin{align*}
 \expect_{\calS}[ (\dotp - y)^2] = \expect_{\calS} \left[ \dotp^2 \right] + \expect_{\calS} \left[ y^2 \right] - 2\langle \bm{\theta},  \meanop \rangle\:\:.
\end{align*}
Taking the minimizers on both sides we obtain
\begin{align*}
\argmin_{\bm{\theta}} \expect_{\calS}[f(\dotp, y)] & = \argmin_{\bm{\theta}} \expect_{\calS} \left[ \dotp^2 \right] - 2 \langle \bm{\theta},  \meanop \rangle \\
& = \argmin_{\bm{\theta}} \| X^\top \bm{\theta} \|^2_2 - 2 \langle \bm{\theta},  \meanop \rangle\:\:.
\end{align*}

\section{The role of \lol s in \citep{pnsCF}}\label{sec:pu}

Let $\pi_+ \defeq \mathbb{P}(y = 1)$ and let $\calD_+$ and $\calD_-$ respectively the set of positive and negative examples in $\calD$. Consider first
\begin{align}
\label{pu-first}
\expect_{(\bm{x}, \cdot) \sim \calD} \left[ l(-\dotp) \right] = \pi_+ \expect_{(\bm{x}, \cdot) \sim \calD_+} \left[ l(-\dotp) \right] + (1 - \pi_+) \expect_{(\bm{x}, \cdot) \sim \calD_-} \left[ l(-\dotp) \right]\:\:.
\end{align}
Then, it is also true that 
\begin{align}
\expect_{(\bm{x}, y) \sim \calD} \left[ l(y\dotp) \right] = \pi_+ \expect_{(\bm{x}, y) \sim \calD_+} \left[ l(y\dotp) \right] + (1 - \pi_+) \expect_{(\bm{x}, y) \sim \calD_-} \left[ l(y\dotp) \right]\:\:. \label{pu:1-2}
\end{align}
Now, solve Equation \ref{pu-first} for $(1 - \pi_+)\expect_{(\bm{x}, y) \sim \calD_-} \left[ l(y\dotp) \right] = (1 - \pi_+)\expect_{(\bm{x}, y) \sim \calD_-} \left[-l(-\dotp) \right] $ and substitute it into \ref{pu:1-2} so as to obtain:
\begin{align}
& \expect_{(\bm{x}, y) \sim \calD} \left[ l(y\dotp) \right] = \nonumber \\
&= \pi_+ \expect_{(\bm{x}, y) \sim \calD_+} \left[ l(y\dotp) \right] + \expect_{(\bm{x}, \cdot) \sim \calD} \left[ l(-\dotp) \right] - \pi_+ \expect_{(\bm{x}, \cdot) \sim \calD_+} \left[ l(-\dotp) \right] \nonumber \\
&= \pi_+ \left( \expect_{(\bm{x}, y) \sim \calD_+} \left[ l(+\dotp) \right] - \expect_{(\bm{x}, \cdot) \sim \calD_+} \left[ l(-\dotp) \right] \right) + \expect_{(\bm{x}, \cdot) \sim \calD} \left[ l(-\dotp) \right] \nonumber \\
&= 2 \pi_+ \expect_{(\bm{x}, y) \sim \calD_+} \left[ l_o(+\dotp) \right]+ \expect_{(\bm{x}, \cdot) \sim \calD} \left[ l(-\dotp) \right]\:\: \nonumber \label{pu:second} 
\end{align}
By our usual definition of $l_o(x) = \frac{1}{2}(l(x) - l(-x))$. Recall that one of the goal of the authors is to conserve the convexity of this new crafted loss function. Then, \citep[Theorem 1]{pnsCF} proceeds stating that when $l_o$ is convex, it must also be linear.  And therefore they must focus on \lol s. The result of \citep[Theorem 1]{pnsCF} is immediate from the point of view of our theory: in fact, an odd function can be convex or concave only if it also linear. The resulting expression based on the fact $l(x) - l(-x) = 2ax$ simplifies into
\begin{align*}
\expect_{(\bm{x}, y) \sim \calD} \left[ l(y\dotp) \right] & =  a \pi_+ \expect_{(\bm{x}, y) \sim \calD_+} \left[ y\dotp \right]+ \expect_{(\bm{x}, \cdot) \sim \calD} \left[ l(-\dotp) \right] \\
&= a \pi_+ \meanop_{\calD_+} + \expect_{(\bm{x}, \cdot) \sim \calD} \left[ l(-\dotp) \right]\:\:.
\end{align*}
where $\meanop_{\calD_+}$ is a mean operator computed on positive examples only. Notice how the second term is instead label independent, although it is not an even function as in our Factorization Theorem.

\section{Additional examples of loss factorization}

\begin{table*}[h!]
\centering
$$
\begin{array}{|l|l|c|c|} \hline
\mbox{}  & \mbox{loss}  & \mbox{even function $l_e$} & \mbox{odd function $l_o$} \\ \toprule \hline 
\mbox{generic}  &  l(x) & \frac{1}{2}(l(x)+l(-x)) & \frac{1}{2}(l(x)-l(-x)) \\ \hline
\mbox{01}  &  1\{x \leq 0\} & 1 - \frac{1}{2} \{x \neq 0\}  & -\frac{1}{2} \sgn(x) \\ \hline
\mbox{exponential}  &  e^{-x} & \cosh(x) & -\sinh(x)\\ \hline 
\mbox{hinge}  &  [1 - x]_+ & \frac{1}{2} ([1 - x]_+ - [1 - x]_+) & \frac{1}{2} ([1 - x]_+ - [1 + x]_+) ^{\dagger} \\  \hline
\mbox{\lol}  &  l(x) & \frac{1}{2}(l(x)+l(-x))  & -ax \\ 
\mbox{$\rho$-loss}  & \rho |x|-\rho x+1 & \rho |x|+1 & -\rho x \:\:(\rho \geq 0) \\ 
\mbox{unhinged}  & 1-x & 1 & -x \\ 
\mbox{perceptron} & \max(0, -x) & x~\sign(x) & -x \\ 
\mbox{2-hinge} & \max(-x, 1/2 \max (0,  1-x)) &  \dagger\dagger   & -x \\
\mbox{\spl}  &  a_l+ l^\star(-x) / b_l& a_l + \frac{1}{2 b_l}(l^\star(x)+l^\star(-x)) & - x / (2 b_l)\\
\mbox{logistic}  & \log (1+e^{-x}) & \frac{1}{2}\log (2+e^x+e^{-x}) & - x /2\\ 
\mbox{square} & (1-x)^2 & 1+x^2 & -2x \\
\mbox{Matsushita} & \sqrt{1+x^2}-x & \sqrt{1+x^2} & -x \\ \bottomrule
\end{array}
$$
\caption{Factorization of losses in light of Theorem \ref{th:factorization2}. $^{\dagger}$The odd term of hinge loss is upperbounded by $(1 - x)/2$ in \ref{proof:upper}. $^{\dagger\dagger} = \max(-x, 1/2 \max (0,  1-x)) + \max(x, 1/2 \max (0,  1+x))$.
\label{table:1}}
\end{table*}

\begin{figure}[h]
\centering
\subfigure[0-1 loss]{\label{fig:01} \includegraphics[width=.35\textwidth]{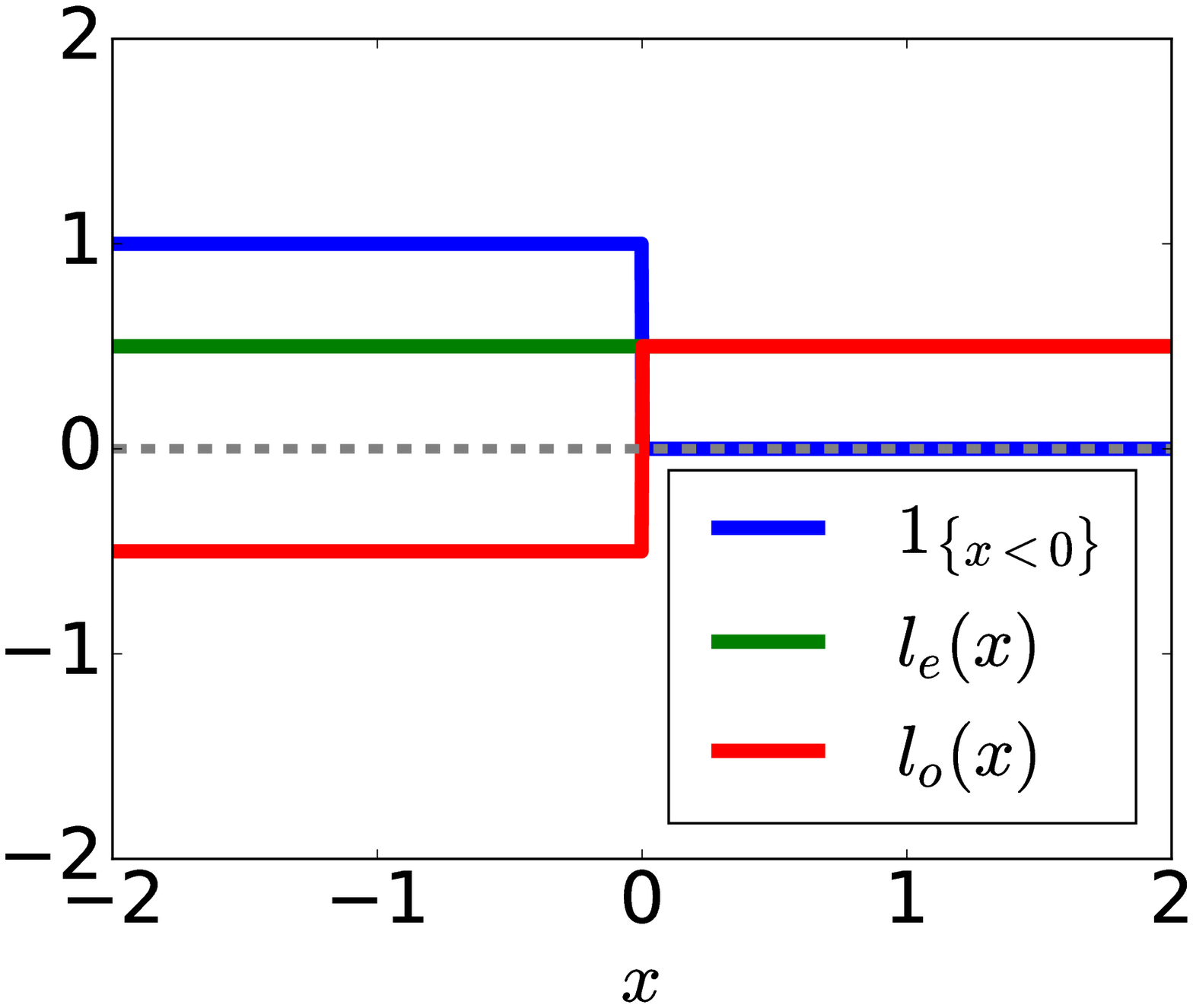}}
\subfigure[Matsushita loss]{\label{fig:matsushita} \includegraphics[width=.35\textwidth]{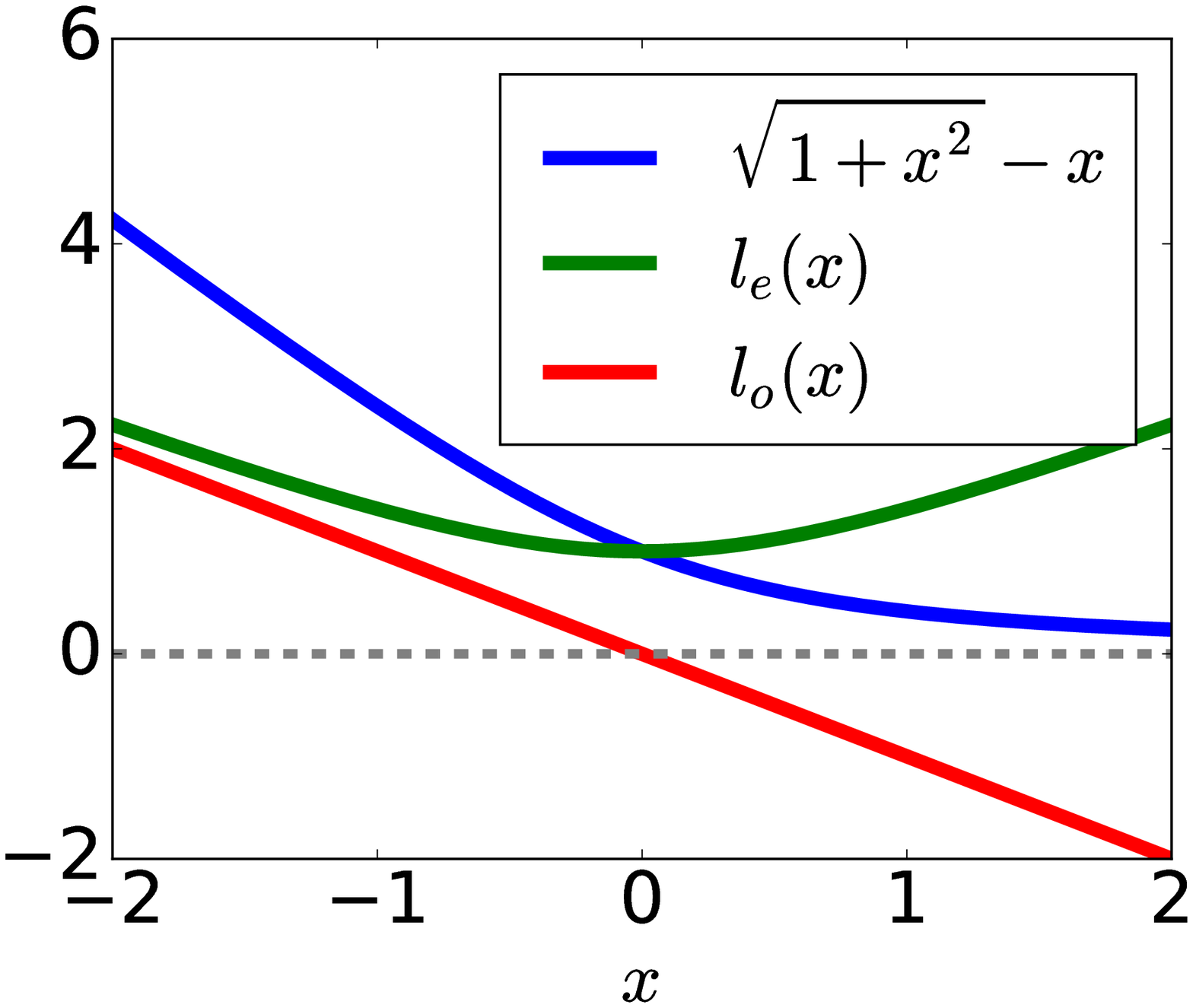}} \\
\subfigure[$\rho$-loss, $\rho$=1]{\label{fig:abs} \includegraphics[width=.35\textwidth]{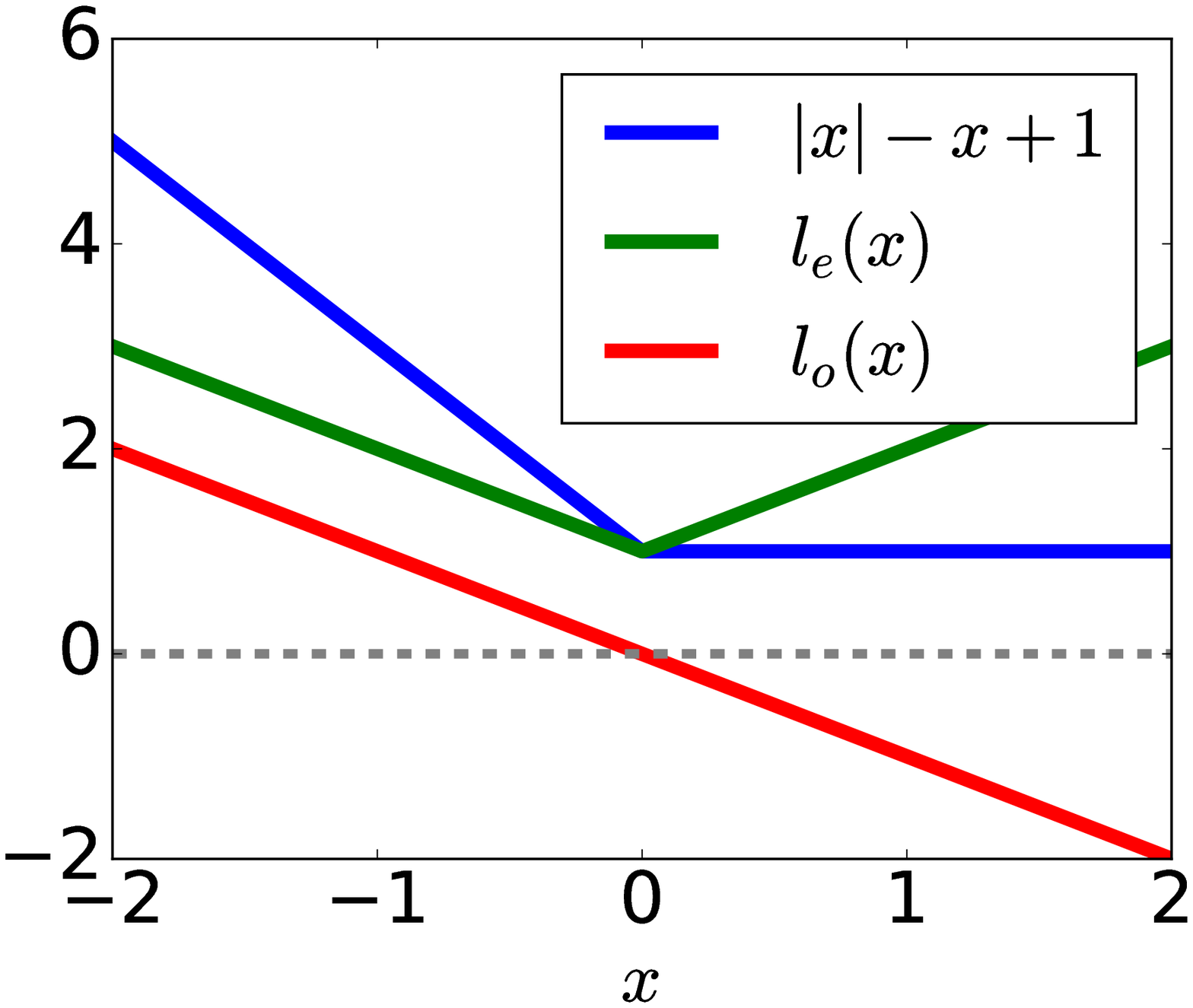}} 
\subfigure[2-hinge loss]{\label{fig:2-hinge2} \includegraphics[width=.35\textwidth]{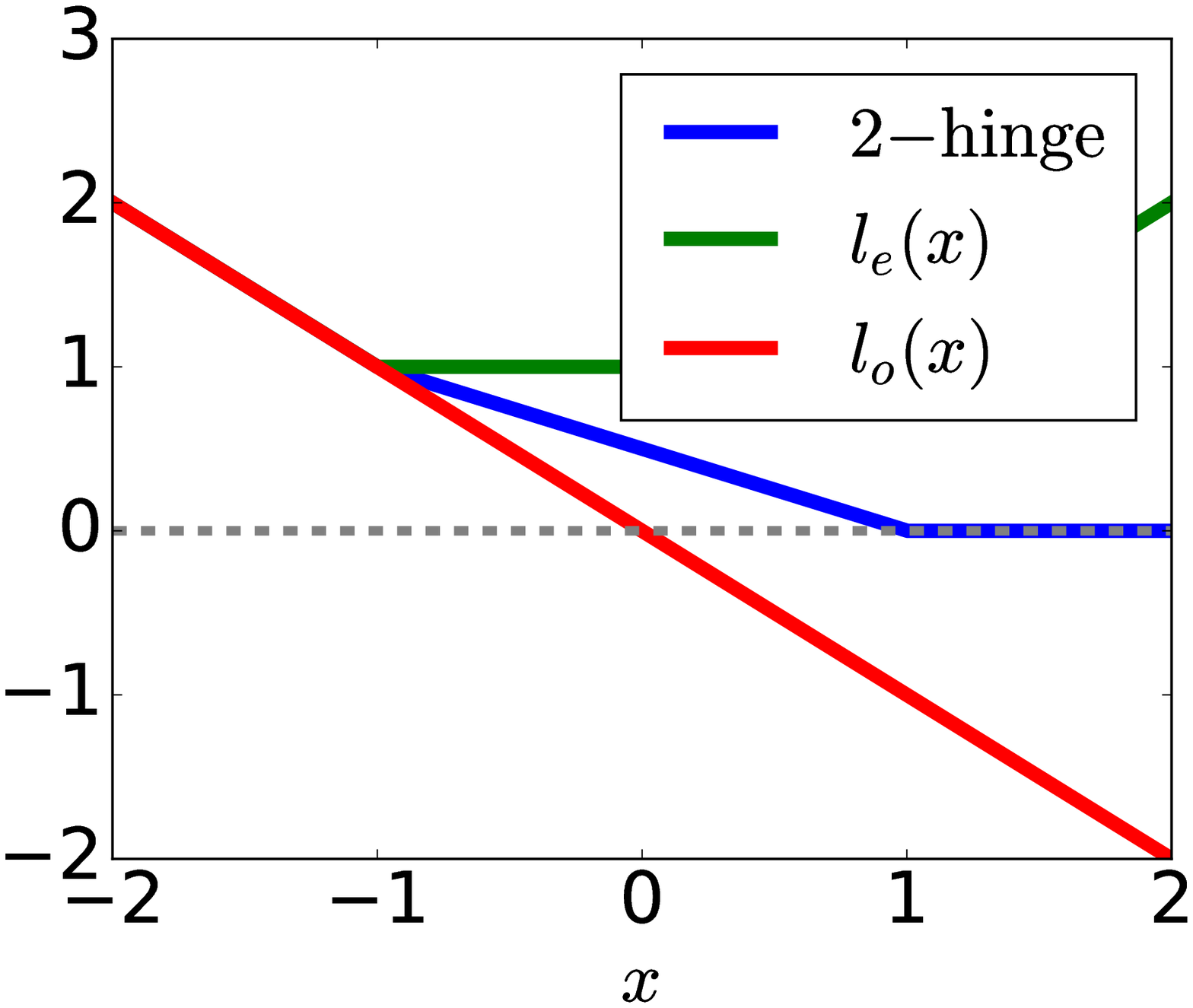}} \\
\subfigure[hinge loss]{\label{fig:hinge} \includegraphics[width=.35\textwidth]{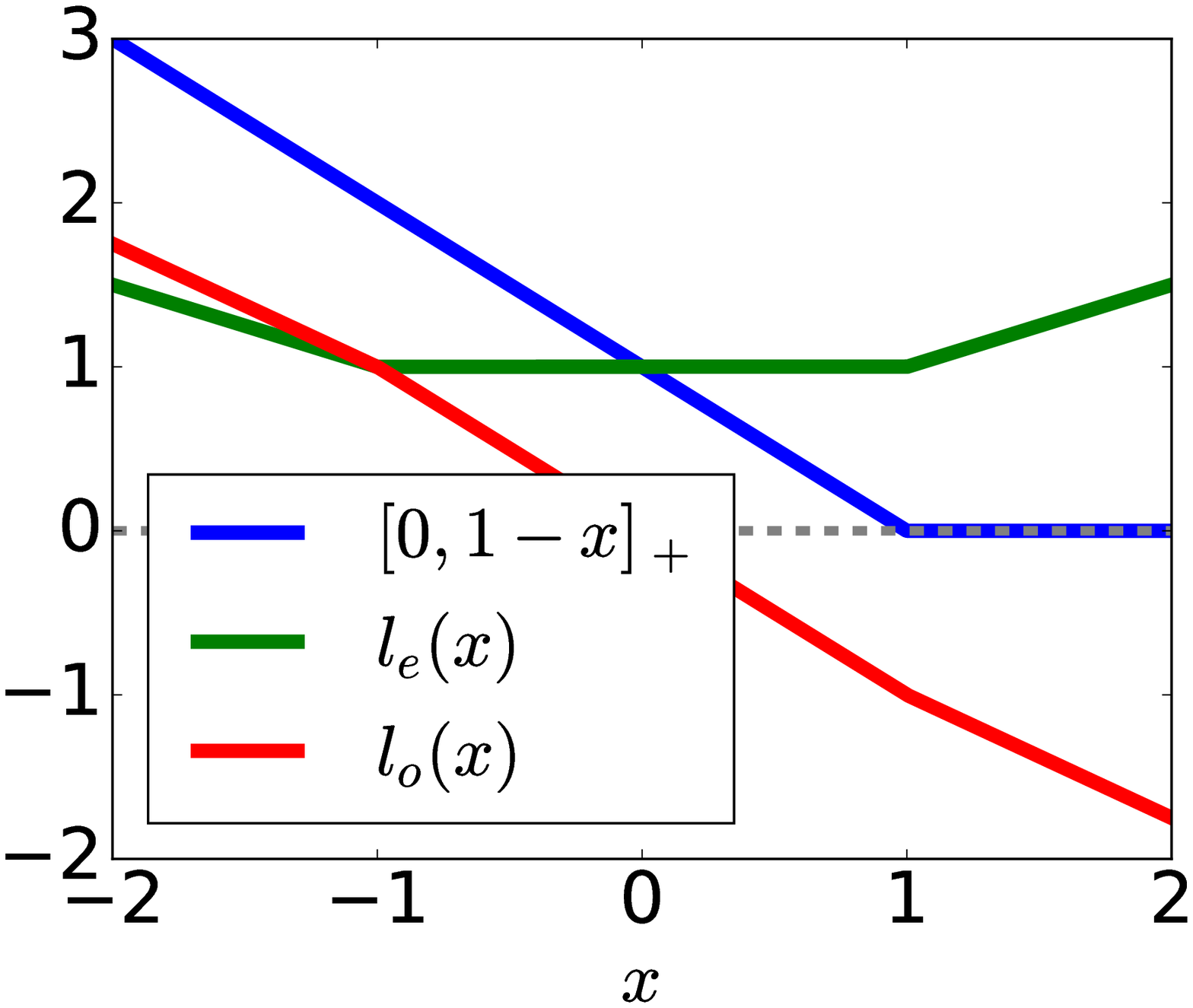}}
\subfigure[Huber loss]{\label{fig:huber} \includegraphics[width=.35\textwidth]{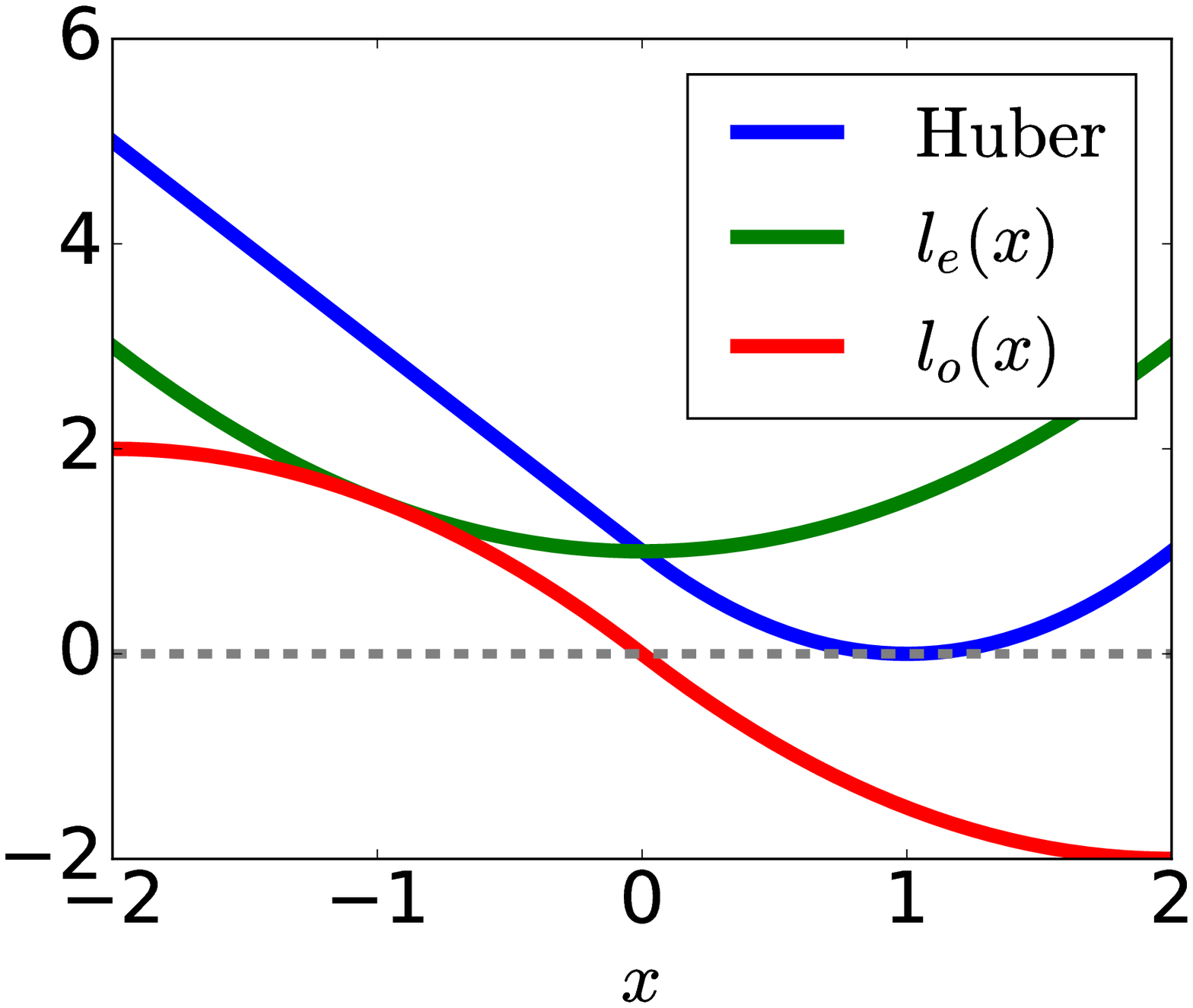}}
\end{figure}

\end{document}